\documentclass[10pt, conference, letterpaper]{IEEEtranx}

\IEEEoverridecommandlockouts
\usepackage{amsmath,amssymb,amsfonts,amsthm}
\usepackage{algorithmic}
\usepackage[ruled,linesnumbered]{algorithm2e}
\usepackage{graphicx}
\usepackage{textcomp}
\usepackage {epstopdf}
\usepackage[hidelinks]{hyperref}
\usepackage{subcaption}
\usepackage{float}
\usepackage[utf8]{inputenc}
\usepackage{tabularx}  
\usepackage{fancyhdr}

\usepackage{xcolor}
\usepackage{authblk}

\usepackage{bm}
\def\BibTeX{{\rm B\kern-.05em{\sc i\kern-.025em b}\kern-.08em
    T\kern-.1667em\lower.7ex\hbox{E}\kern-.125emX}}

\title{Joint Optimization of Energy Consumption and Completion Time in Federated Learning\thanks{This research is part of the programme DesCartes and is supported by the National Research Foundation, Prime Minister’s Office, Singapore under its Campus for Research Excellence and Technological Enterprise (CREATE) programme. 
This research is also supported by the National Research Foundation, Singapore under its Strategic Capability Research Centres Funding Initiative, Nanyang Technological University (NTU)
Startup Grant, Nanyang Technological University (NTU)-Wallenberg AI, Autonomous Systems and Software Program (WASP) Joint Project, and Singapore Ministry of Education Academic Research Fund under Grant Tier 2 MOE2019-T2-1-176. Any opinions, findings and conclusions or recommendations expressed in this material are those of the author(s) and do not reflect the views of National Research Foundation, Singapore. Corresponding Author: Jun Zhao.}} 
\date{December 2020}

\usepackage{graphicx}
\usepackage{mathrsfs}

\newtheorem{theorem}{Theorem}
\newtheorem{lemma}{Lemma}

\renewenvironment{thebibliography}[1]{
  \begin{oldthebibliography}{#1}
    \setlength{\itemsep}{0em}
    \setlength{\parskip}{0em}
}
{
  \end{oldthebibliography}
}

\begin{document}

\author[1]{Xinyu Zhou}
\author[2]{Jun Zhao}
\author[2]{Huimei Han}
\author[3]{Claude Guet}

\affil[1]{\small  School of Computer Science \& Engineering and ERI@N, Nanyang Technological University, Singapore}
\affil[2]{\normalsize School of Computer Science \& Engineering, Nanyang Technological University, Singapore} 
\affil[3]{\normalsize School of Physical \& Mathematical Sciences, Nanyang Technological University, Singapore} 
\affil[ ]{\normalsize xinyu003@e.ntu.edu.sg, junzhao@ntu.edu.sg, huimei.han@ntu.edu.sg, cguet@ntu.edu.sg} 



 \pagestyle{empty} \thispagestyle{empty}

 \maketitle 
 \thispagestyle{fancy}
\pagestyle{fancy}
\renewcommand{\footrulewidth}{0pt}
\chead{2022 IEEE 42nd International Conference on Distributed Computing Systems (ICDCS)}
\renewcommand{\headrulewidth}{0pt} 

\renewcommand{\footrulewidth}{0pt}


\begin{abstract}
Federated Learning (FL) is an intriguing distributed machine learning approach due to its privacy-preserving characteristics. 
To balance the trade-off between energy and execution latency, and thus accommodate different demands and application scenarios, we formulate an optimization problem to minimize a weighted sum of total energy consumption and completion time through two weight parameters.
 The optimization variables include bandwidth, transmission power and CPU frequency of each device in the FL system, where all devices are linked to a base station and train a global model collaboratively.
Through decomposing the non-convex optimization problem into two subproblems, we devise a resource allocation algorithm to determine the bandwidth allocation, transmission power, and CPU frequency for each participating device. We further present the convergence analysis and computational complexity of the proposed algorithm.
Numerical results show that our proposed algorithm not only has better performance at different weight parameters (i.e., different demands) but also outperforms the state of the art.
\end{abstract}


 \setcounter{page}{1005}

\section{Introduction}


 \setlength{\belowdisplayskip}{4pt plus 1pt minus 1.0pt} \setlength{\belowdisplayshortskip}{4pt plus 1pt minus 1.0pt}
\setlength{\abovedisplayskip}{4pt plus 1pt minus 1.0pt} \setlength{\abovedisplayshortskip}{0.0pt plus 2.0pt}





To create a model with high accuracy, machine learning (ML) needs a large size of high-quality training data. Due to limitations of user privacy, data security and competition among industries, usually the data is limited. Besides, the concern around data privacy is rising.  Related laws and regulations are appearing, such as the General Data Protection Regulation  (GDPR) of the European Union  (EU), so it becomes more challenging to collect sensitive data together for learning a model. Therefore, federated learning (FL) could be a solution to collaboratively train an ML model while preserving data on its local device distributedly~\cite{li2021survey}. FL is a special distributed learning framework, which was proposed by Google in \cite{mcmahan2017communication}. The algorithm in \cite{mcmahan2017communication} was called FedAvg. In FedAvg, each device trains a model locally for a certain number of local iterations and uploads their parameters to a base station periodically to build a shared global model. 

\textbf{Challenges and motivation}. On account of the huge number of parameters in prevalent ML models, FL has high requirements of device memory, computational capability and communication bandwidth.
The limited communication and computation resources in the real world bring several challenges to the deployment of FL: (1) Limited bandwidth results in long \textit{latency} between clients and the server and consequently affects the convergence time of the global model. (2) A large amount of \textit{energy} is needed because a satisfactory model requires plenty of local computation and communication rounds. (3) Due to the large size of current ML models, low-battery devices with limited computation resources have difficulties in joining an FL framework \cite{luo2020hfel}. To improve the efficiency of FL,
it is crucial to save both communication and computation resources, which can be achieved by reducing communication overhead~\cite{luping2019cmfl,han2020adaptive,chen2021communication} or finding optimal resource (e.g., transmission power, bandwidth, CPU frequency, etc.) allocation strategies~\cite{yang2021energy, luo2020hfel, dinh2021federated}. This paper focuses on the latter solution.
An optimal resource allocation strategy will allow us to reduce the execution latency and energy consumption, and accelerate the convergence of the global model. 
Therefore, it is essential to consider how to adjust bandwidth, transmission power and CPU clock frequency for each participating device to jointly optimize energy consumption and completion time (i.e., delay). Moreover, different scenarios have different energy and latency demands. For example, connected autonomous vehicles in a smart transportation system require low-latency control decisions, while personal computers may tolerate higher latency to utilize the idle time to train their model. Thus, it raises the following question: how to allocate the transmit power, bandwidth and CPU frequency to each user device to find the optimal trade-off between energy consumption and the total completion time of the FL system while accommodating to different scenarios?


Motivated by the above question, we consider a basic FL system via frequency-division multiple access (FDMA) and formulate a joint optimization of energy and delay problem over a wireless network. 
The objective function includes a weighted sum of the total energy consumption and the total completion time. Given a certain weight for a specific application scenario, we can find a trade-off between energy and execution latency.
The participating devices in FL are under power and bandwidth constraints while pursuing as little energy and time consumption as possible. 

\textbf{Contributions}. In summary, the main contributions of this paper are listed as follows:
\begin{itemize}
    \item 
    {\color{black} We introduce the joint optimization of energy and delay framework into the design of a federated learning system via FDMA.} The objective is to optimize a weighted sum of energy consumption and total completion time so that it can satisfy diverse requirements.
    \item The optimization problem includes a sum-of-ratio form of communication energy consumption, which is \mbox{non-convex} and cannot be solved through traditional optimization techniques. To tackle this problem, we first propose to decompose the original problem into two subproblems: the first subproblem is to minimize computation energy and completion time, and the second is to minimize the communication energy. The first one can be solved through its Lagrangian dual problem. The second subproblem is a sum-of-ratio minimization problem, and we transform it into a convex subtractive-form problem through a Newton-like method similar to \cite{jong2012efficient} used for solving a general fractional programming problem. Convergence analysis and time complexity are provided. 
    \item Numerical results show: 1) at different weight parameters (i.e., different requirements), our resource allocation algorithm is able to find a corresponding solution to jointly minimize energy consumption and completion time. 2) Our joint optimization framework's performance is better than the communication-only optimization and computation-only optimization framework. 3) Our resource allocation scheme outperforms the state of the art, especially when there is a tight delay constraint.
\end{itemize}

The rest of this paper is organized as follows. Section \ref{sec:literature_review} presents related works. In Section \ref{sec:sys_model}, we introduce the FL model and the necessary notations. Problem formulation is provided in Section \ref{sec:prob_formu}. We detail the proposed algorithm in Section \ref{sec:solution_joint_prob}. Then Section~\ref{sec:conv} provides the convergence and the computational complexity of the algorithm. 
 Numerical results are shown in Section \ref{sec:experiments}. Finally, Section VIII concludes the paper.
 

\section{Related Work} \label{sec:literature_review}

In this section, we survey related research. Section~\ref{sec:literature_review1} as the first subsection is devoted to federated learning (FL) over wireless networks which our paper is also about. Then in Section~\ref{sec:literature_review2}, we elaborate on other FL studies. Finally, since our paper tackles resource allocation and appears in ICDCS, Section~\ref{sec:literature_review3} presents related works on resource allocation in other distributed systems.

\subsection{Federated Learning over Wireless Networks} \label{sec:literature_review1}

For simplicity, below we use WFL, short for wireless federated learning, to describe FL over wireless networks. In many papers~\cite{li2021talk,dinh2021federated,wang2019adaptive,zeng2021energy,luo2020hfel}, the base station (i.e., the aggregator) is also seen as an edge server, so the literature may also use the term  federated edge learning (FEEL). Since our paper minimizes the weighted sum of two metrics: energy consumption and delay, our coverage about WFL next is categorized into energy minimization, delay minimization, optimization of other metrics, weighted optimization, and \mbox{non-optimization}  studies, respectively. Note that papers~\cite{dinh2021federated,luo2021cost,luo2020hfel} minimizing the weighted sum of energy and delay are elaborated in the part about weighted optimization. During the discussion, we also compare our work with many studies including \cite{yang2021energy,yang2020delay,li2021talk,zeng2021energy,chen2020convergence,dinh2021federated,luo2021cost,luo2020hfel}.  



\textbf{Minimizing the energy consumption in WFL.} Three representative papers~\cite{yang2021energy,zeng2021energy,li2021talk} minimize the  energy consumption only in WFL are explained as follows.

Yang~\textit{et~al.}~\cite{yang2021energy} minimize the sum of all devices' energy consumption, which includes the local computation energy and the wireless transmission energy. In \cite{yang2021energy}, devices access the wireless transmission media via FDMA, as in our paper. The optimization variables in \cite{yang2021energy} include the local training accuracy and each device's transmission time per global round,  bandwidth allocation, CPU frequency and transmit power. The constraints cover an upper bound on the total delay and a lower bound on each device's transmission time per global round such that the local model can be uploaded.

The \textbf{major difference between \cite{yang2021energy}  and our work} is that the optimization objective is just the energy consumption in~\cite{yang2021energy}, but is the weighted sum of the energy and delay in our paper. The consideration of both energy and delay in the objective function makes our optimization much more difficult to solve than~\cite{yang2021energy}, as explained below. {\color{black}Since we put the total completion time into the objective function, a part of our optimization problem (i.e., Subproblem 2 in (\ref{equa:min6}) later) becomes an NP-complete problem---the sum-of-ratios problem \cite{jong2012efficient}}. While addressing a weaker optimization objective than our work,~\cite{yang2021energy} has two more optimization variables than ours: the local training accuracy and each device's transmission time per global round. Yet, since both two variables can be decoupled from others, their optimal values can be easily obtained as in Theorem 2 of~\cite{yang2021energy}. {\color{black} The rest variables of \cite{yang2021energy} are the same as ours. However, if just minimizing the total energy consumption, it will be easy to decouple the optimization problem into two subproblems as \cite{yang2021energy} does. Since all subproblems in \cite{yang2021energy} are convex, \cite{yang2021energy} can easily obtain the optimal solution.

}




Different from~\cite{yang2021energy} and our paper which both consider only CPU computing resources at devices, Li~\textit{et~al.}~\cite{li2021talk} and
Zeng~\textit{et~al.}~\cite{zeng2021energy} incorporate devices' GPU computation for WFL. In particular,~\cite{li2021talk} studies GPU only, while~\cite{zeng2021energy} uses CPU and GPU simultaneously. The energy consumption model for GPU is different from that of CPU used in our paper, so we omit the detailed discussions of~\cite{li2021talk,zeng2021energy}.

\textbf{Minimizing the delay in WFL.} Below we discuss three work~\cite{yang2020delay,chen2020convergence,luo2020hfel} on delay minimization only in WFL.

Minimizing the delay by 
Yang~\textit{et~al.}~\cite{yang2020delay} is actually used to provide an initial feasible solution for the problem of minimizing the energy consumption subject to the delay constraint in~\cite{yang2021energy}. In other words, the work done in~\cite{yang2020delay} is a subroutine of that in~\cite{yang2021energy}. We have discussed above how our optimization is more  challenging to solve than that of~\cite{yang2021energy}, so our work is further more difficult than~\cite{yang2020delay}. In~\cite{yang2020delay}, the FDMA bandwidth allocation is not analyzed directly, but a bisection method is used to sandwich the minimum delay.


Distinct from~\cite{yang2020delay} considering FDMA just like its extended version~\cite{yang2021energy}, the research~\cite{chen2020convergence} by Chen~\textit{et~al.} employs orthogonal frequency-division multiple access (OFDMA) for delay minimization in WFL. In OFDMA, the radio is divided into 2-dimensional resource blocks (RBs) over time and frequency. The optimization variables in~\cite{chen2020convergence} include device selection matrices and RB allocation matrices, which all have binary elements. Hence,~\cite{chen2020convergence}  handles binary integer programming, while the variables in~\cite{yang2020delay} and our paper are all continuous.







\textbf{Optimizing other metrics in WFL.} Other metrics in WFL have also been optimized, such as 
the number of global aggregation rounds~\cite{zeng2020federated},
the total number of bits sent from the devices to the aggregator~\cite{wang2021edge}, the admission data rates of all devices~\cite{wang2021cflmec},
the
communication error~\cite{9488818},
the total communication cost~\cite{deng2021share} with each communication link being assigned a cost number,
the training loss~\cite{wang2019adaptive,chen2020joint,shi2020joint},
the learning efficiency~\cite{he2020importance,ren2020accelerating}, and
a metric capturing the long-term performance of WFL~\cite{xu2020client}.


\textbf{Weighted optimization in WFL.}
Below we review three work~\cite{dinh2021federated,luo2021cost,luo2020hfel} on the weighted minimization of energy consumption and delay in WFL.

\textbf{The major difference between \cite{dinh2021federated} and our paper} is that the channel access method in \cite{dinh2021federated} for devices is time-division multiple access (TDMA), while our paper adopts FDMA. In addition, although the optimization variables in \cite{dinh2021federated} also include parameters capturing local training and global training progresses respectively, the two parameters are optimized just by a numerical method: an exhaustive grid search. Other optimization variables in \cite{dinh2021federated} include each device's CPU frequency and transmission time per global round (we do not discuss additional variables listed in \cite{dinh2021federated} since they can be decided from variables already explained). 


In~\cite{luo2021cost}, the optimization variables include the number of local iterations per global round,  the number of global rounds, and the number of devices selected (via uniform samling) per global round. Hence,~\cite{luo2021cost} tackles integer programming, while the variables in~\cite{dinh2021federated} and our paper are all continuous.


In contrast to~\cite{dinh2021federated,luo2021cost} which have only a two-level hierarchy: the base station (often also an edge server) and the devices,~\cite{luo2020hfel} adopts a three-level hierarchy (introduced also by~\cite{9148862}): an cloud server, multiple edge servers, and the devices. \textbf{A detailed comparison between~\cite{luo2020hfel} and our paper is as follows.} Firstly,
the  optimization variables in both work include each device's bandwidth allocation via FDMA and CPU frequency, while~\cite{luo2020hfel} also considers edge association of devices, and our paper also optimizes each device's transmission power. Secondly, in our work, a challenge in solving the optimization arises from the coupling between the FDMA bandwidth allocation and the weighted minimization of energy and delay. 
Yet,~\cite{luo2020hfel} avoids such challenge by simplifying the Shannon formula for devices' data rates. In particular, in Equation (5) of~\cite{luo2020hfel}, the noise power as the denominator inside the logarithm of the Shannon formula is forcefully assumed as a constant that does not scale with the allocated bandwidth, while it should the noise power spectral density times the bandwidth. Such extreme simplification in~\cite{luo2020hfel} makes the  bandwidth allocation easy to solve. In this paper, we use the exact expression of the Shannon formula and hence run into a much more challenging optimization.

The work~\cite{feng2021design} also considers weighted optimization in WFL, but is not about
energy consumption or delay.



\textbf{Studies other than optimizing metrics in WFL.} Beyond optimization studies, WFL has also been investigated in terms of 
scheduling policies~\cite{yang2019scheduling,ren2020scheduling,wang2020towards}, 
 and mobility-based orchestration~\cite{deveaux2020orchestration}.



\subsection{Other Federated Learning Studies}  \label{sec:literature_review2}

In addition to federated learning over wireless networks elucidated in the previous subsection, we present other federated learning work in this subsection.

\textbf{Communication-efficient FL.} Researchers have enabled FL to     be more communication-efficient, by transmitting clustered model updates~\cite{9448151}, avoiding irrelevant updates~\cite{luping2019cmfl}, adaptive parameter freezing~\cite{chen2021communication}, and sketching~\cite{rothchild2020fetchsgd}.

\textbf{Security/privacy of FL.} Security of FL has been addressed by defending against Byzantine attacks~\cite{guo2021siren}, and enforcing  device liability~\cite{malandrino2021toward}. Privacy of FL is examined in~\cite{9751555,wei2020federated}.

\textbf{Incentives in FL.} Designing incentive mechanisms for FL has been tackled in~\cite{9355731,zhou2021truthful}.

\textbf{Machine learning approaches to improve FL.} FL systems have also benefited from other machine learning techniques including reinforcement learning~\cite{9155494}, meta learning~\cite{yue2022efficient,9355664}, one-shot learning~\cite{DBLP:journals/corr/abs-2009-07999}, and online learning~\cite{han2020adaptive}.

\textbf{Applications of FL.} FL has been applied to various tasks including
 crowdsensing~\cite{wang2020learning,wang2021federated,liu2020boosting}, RFID clone detection~\cite{piva2021tags}, and the classification of unmanned aerial vehicles (UAVs)~\cite{9432739}.

\subsection{Resource Allocation in Other Distributed Systems}  \label{sec:literature_review3}

Resource allocation has also been studied in other distributed systems, such as blockchain~\cite{huang2021resource,DBLP:journals/corr/abs-2106-12332},  
cloud computing~\cite{276950,9036983}, edge computing~\cite{9491614}, distributed  learning~\cite{10.1145/3363554,9491589}, and wireless networks~\cite{9547355}.




\begin{table}  
\caption{Notations}  
\label{tab:notation}
\begin{tabularx}{8.5cm}{@{}l@{\hspace{.9\tabcolsep}}lX}  
\hline                      
Description & Symbol  \\  
\hline  
Total energy consumption  & $E$ \\   
Maximum uplink bandwidth & $B$ \\
Allocated bandwidth of device $n$ & $B_n$ \\
The number of samples on device $n$ & $D_n$ \\  
Total number of samples on all devices & $D$ \\  
Transmission data size of device $n$ & $d_n$ \\  
Maximum completion time in one global communication round  & $\mathcal{T}$ \\  
Maximum completion time  & $T$ \\  
Data transmission time of device $n$ & $T^{up}_n$ \\
Local computation time of device $n$ & $T^{cmp}_n$ \\
Data transmission rate of device $n$ & $r_n$ \\
Minimum transmission rate of device $n$ & $r_n^{min}$ \\
CPU frequency of device $n$ & $f_n$         \\
Minimum and maximum CPU frequency of device $n$ & $f_n^{min}, f_n^{max}$ \\
The number of global aggregation rounds & $R_g$ \\ 
The number of local iteration per global round & $R_l$ \\ 
The  effective switched capacitance & $\kappa$  \\ 
The number of CPU cycles per sample on device $n$ & $c_n$     \\ 
Local computation energy consumption of device $n$ & $E_n^{cmp}$  \\ 
Transmission energy consumption of device $n$  & $E_n^{trans}$ \\ 
The number of total devices & $N$ \\
Weight parameters & $w_1$, $w_2$ \\ 
Noise power spectral density & $N_0$ \\
Transmission power of device $n$ & $p_n$ \\
Minimum and maximum transmission power of device $n$ & $p_n^{min}, p_n^{max}$ \\
Channel gain & $\bm{g}$ \\
Lagrangian multipliers & $\bm{\lambda},\bm{\tau},\mu$ \\
Auxiliary variables & $\bm{\nu},\bm{\beta}$\\
\hline
\end{tabularx}  
\end{table}

\section{System Model} \label{sec:sys_model}
The system model is FL over wireless communications. We consider a star network with a base station, as shown in Fig. \ref{fig:FL_framework}.
We suppose that the base station serves $N$ devices. The set of $N$ devices is denoted by $\mathcal{N}$. 

{\color{black} \textbf{Federated Learning.}}
Each device contains a dataset $\mathcal{D}_n$ with $D_n$ samples. 
The total dataset is $\mathcal{D}$ = $\bigcup_{n=1} ^{N}\mathcal{D}_n$. 
We have the local dataset $\mathcal{D}_n$ = $\bigcup_{i=1}^{D_n}\{\boldsymbol{x_{ni}},y_{ni}\}$. 
$\boldsymbol{x_{ni}} \in \mathbb{R}^p$ is an input vector containing the features of each data sample, where $\mathbb{R}$ denotes the set of real numbers.
$y_{ni}$ is the label of each data sample. 
Notations in this paper are listed in Table~\ref{tab:notation}. Bold symbols stand for vectors.

\begin{figure}[h]
    \centering
    \includegraphics[scale=0.5]{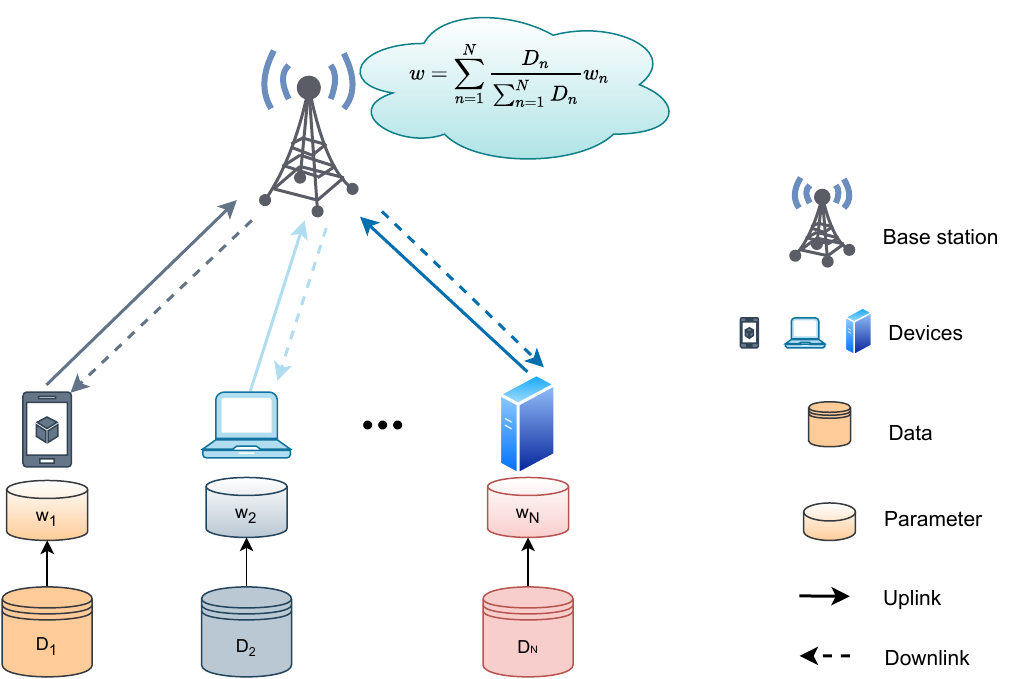}
    \caption{Federated learning.}
    \label{fig:FL_framework}
\end{figure}

Then, the overall FL model is to solve the problem $\min _{\boldsymbol{w}\in\mathbb{R}^p} F(\boldsymbol{w}) = \sum_{n=1}^{N} \frac{D_n}{D} l_n(\boldsymbol{w})$,
 where $D = \sum_{n=1}^{N} D_n$. 

Each local device trains its own model and the aim is to minimize the corresponding loss function defined at device $n$ as $l_n(\bm{w}) = \sum_{i=1}^{D_n} \frac{f_{ni}(\bm{w})}{D_n},$
 where $\bm{w}$ is {\color{black} the parameters related to the whole global model. 
The function $f_{ni}(\bm{w})$ measures the error between the model prediction and the ground truth label. Note that in each local iteration, each device $n$ uses all of its $D_n$ data samples, not just a subset.}

As shown in Fig.~\ref{fig:FL_framework}, each device trains its own model locally for a certain number of iterations and then uploads its model parameter to the base station. The base station uses all the uploaded parameters to calculate the {\color{black} weighted}
 average model parameter and broadcasts it to each device, {\color{black} where the weight for each device $n$'s parameter is $\frac{D_n}{D}$; i.e., the fraction of device $n$'s number of data points over the total number of data points}. The above uploading and broadcasting process is called global communication.

{\color{black} \textbf{FDMA.}} In this paper, different devices communicate with the base station via FDMA. FL over FDMA has also been studied in~\cite{yang2020delay,li2021talk,yang2021energy,luo2020hfel} (see Section \ref{sec:literature_review} for detailed comparisons between these work and ours). FDMA is simple to implement even for mobile devices with limited computation capability~\cite{li2021talk}. We do not consider OFDMA since our optimization problem with FDMA is already difficult to tackle. Although~\cite{chen2020convergence} adopts OFDMA for FL, but it addresses
delay minimization only and handles binary integer programming, while the variables in our paper are all continuous. In addition, TDMA over FL is investigated in~\cite{dinh2021federated,he2020importance}.

We will formulate a joint optimization of energy and completion time problem in FL via FDMA. The objective function is a weighted sum of total energy and time. Total energy consumption and total completion time are discussed in Sections~\ref{subsec:total_e_def} and~\ref{subsec:total_t_def}, respectively.

\subsection{Total Energy Consumption} \label{subsec:total_e_def}
The \textbf{total energy consumption}, $E$, comprises wireless network \textbf{transmission energy consumption} and \textbf{local computation energy consumption}. Following works \cite{dinh2021federated, luo2020hfel, yang2021energy}, we do not consider the energy consumption of the base station (or the central server), either.

\subsubsection{Transmission energy consumption}
Generally, transmission time refers to the uplink time and the downlink time. However, since the output power of the base station is significantly larger than the maximum uplink transmission power of a user device, the downlink time is much smaller than the uplink time. Thus, we neglect the downlink time.

From the Shannon formula, we express the data transmission rate, $r_n$, as
\begin{align} \label{def:data_rate}
    r_n = B_n\log_2 (1+ \frac{g_np_n}{N_0B_n}),
\end{align}
where $B_n$ is the bandwidth allocated to user $n$, $N_0$ is the power spectral density of Gaussian noise, $g_n$ is the channel gain between user device $n$ and the base station and $p_n$ is the transmission power. Note that there is no interference in FDMA since the channels assigned to users are separated by different frequencies.

{Strictly speaking, the expression (\ref{def:data_rate}) describes just the maximum theoretical bit rate achievable in a wireless channel. The actual data rate depends on the spectral efficiency of the communication technology. Yet, we use (\ref{def:data_rate}) as the data rate for simplicity as in much prior work~\cite{yang2020delay,li2021talk,yang2021energy,dinh2021federated,he2020importance} on FL. Also, if ``$=$'' in  (\ref{def:data_rate}) is replaced by ``$\leq$'', ``$=$'' will still be taken in our optimization problem to save the transmission time and hence the transmission energy.}

The data transmission will happen in each global aggregation. If the transmission data size of user device $n$ is $d_n$, the transmission time of user device $n$ is
\begin{align} \label{t_up_n}
    T^{up}_n = d_n/r_n.
\end{align}
Then, the \textbf{transmission energy consumption} of user device $n$ is 
\begin{align} \label{e_trans_n}
    E^{trans}_n = p_nT^{up}_n.
\end{align}

\subsubsection{Local computation energy consumption}
Following~\cite{dinh2021federated}, and assuming that the CPU frequency $f_n$ of device $n$ does not change as time goes once it is optimized, we express the energy, $E_n^{cmp\prime}$, consumed in one \textbf{local} iteration of user device $n$ as: 
\begin{align}\label{equa:energy consump}
    E_n^{cmp\prime} = \kappa c_n D_n f_n^2,
\end{align}
where $c_n$ is the number of CPU cycles per sample on device $n$, and $\kappa$ is the effective switched capacitance.

The \textbf{local computation energy consumption in one global iteration} is
\begin{align} \label{e_cmp_n}
    E_n^{cmp} = \kappa R_l c_nD_n f_n^2,
\end{align}
where $R_l$ is the number of local iterations
and the total energy consumption is 
\begin{align} \label{equa:total_energ}
    E = R_g\sum_{n=1}^{N} (E_n^{trans} + E_n^{cmp}),
\end{align}
where $R_g$ is the number of global rounds.

\subsection{Total Completion Time} \label{subsec:total_t_def}
The total completion time (i.e., the total delay), $T$, of the FL system model over wireless network comprises local computation time and data transmission time. The local computation time of device $n$ in one \textbf{global} iteration, $T^{cmp}_n$, is 
\begin{align} \label{t_cmp_n}
    T_n^{cmp} = R_l\frac{c_nD_n}{f_n}.
\end{align}
Then with the transmission time given by $T_n^{up} = \frac{d_n}{r_n}$, the total completion time $T$ becomes $R_g\max_{n \in \mathcal{N}} \{T_n^{cmp} + T_n^{up}\}$.

As we have mentioned, since the base station's output power is larger than the maximum uplink transmission power of a user device, the downlink time is much smaller than the uplink time. Hence, we ignore the downlink time.

\section{Optimization Problem Formulation} \label{sec:prob_formu}
To minimize the total energy consumption as well as the total completion time of the whole FL training process, we use a weighted sum method to formulate the optimization problem. As $E_n^{cmp}$ in (\ref{e_cmp_n}), $E_n^{trans}$ in (\ref{e_trans_n}), $T_n^{up}$ in (\ref{t_up_n}) and $T_n^{cmp}$ in (\ref{t_cmp_n}) are highly related to $p_n$, $B_n$ and $f_n$, we choose $p_n$, $B_n$ and $f_n$ as three optimization variables. Some works may consider more optimization variables such as the global communication rounds $R_g$, the local computations rounds $R_l$ or the model accuracy. However, $R_g$, $R_l$ and the model accuracy are determined by the exact ML model run on each device. To make our optimization problem apply to flexible scenarios and broad use cases, we do not consider more optimization variables. 

The optimization problem is as follows:
\begin{subequations} \label{equa:min}
\begin{align}
\min_{p_n,f_n, B_n} &~ w_1E+w_2R_g\max_{n \in \mathcal{N}} \{T_n^{cmp} + T_n^{up}\}, \tag{\ref{equa:min}} \\
    \text{subject to:} &  \notag \\
    & p_{n}^{min} \le p_n \le p_n^{max},~\forall n\in \mathcal{N}, \label{pminpmax} \\
    & f_{n}^{min} \le f_n \le f_n^{max},~\forall n\in \mathcal{N}, \label{fminfmax} \\
    & \sum _{n=1}^N B_n \leq B \label{constra:bandwidth},
\end{align}
\end{subequations}
with $w_1$, $w_2$ $\in [0, 1]$ denoting two weight parameters of total energy consumption and completion time and $w_1 + w_2 = 1$. If devices are in a low-battery condition, we can set $w_1 = 1$ and $w_2 = 0$ to allocate resources in the FL system. In contrast, if the FL system is time-sensitive, the weight parameters can be set as $w_1 = 0$ and $w_2 = 1$. $E$ is defined in (\ref{equa:total_energ}).
Constraint (\ref{pminpmax}) and constraint (\ref{fminfmax}) set the ranges of the transmission power and computation frequency of each device.
 Constraint (\ref{constra:bandwidth}) refers to the uplink bandwidth constraint. 
 
 Note that the max function in (\ref{equa:min}) makes the optimization hard to solve. To circumvent this difficulty, we replace the max function by an auxiliary variable {\color{black}$\mathcal{T}$} with an additional constraint; then,  the problem writes as:
\begin{subequations} \label{equa:min2}
\begin{align}
\min_{p_n,f_n, B_n, \mathcal{T}} &~ w_1R_g \sum_{n=1}^{N} (p_nT_n^{up} \!+\! \kappa R_lc_nD_nf_n^2) \!+\! w_2 R_g\mathcal{T}, \tag{\ref{equa:min2}} \\
    \text{subject to:} &  \notag \\
& \text{(\ref{pminpmax}),~(\ref{fminfmax}),~(\ref{constra:bandwidth}),}  \notag\\ 
    & T_n^{cmp} + T_n^{up} \leq \mathcal{T}, ~\forall n\in \mathcal{N}.\label{constra:time}  
\end{align}
\end{subequations}
Constraint (\ref{constra:time}) is the total time constraint.

\textbf{Difficulty of solving problem (\ref{equa:min2})}. Since the formulation of $T_n^{up}$ and constraint (\ref{constra:time}) are not convex, problem (\ref{equa:min2}) is non-convex and thus can be difficult to solve. Consequently, we must explore other ways to transform this problem to make it solvable. Additionally, the complex problem formulation increases the difficulty of transforming this problem.

\subsection{Problem Decomposition}
To continuously transform problem (\ref{equa:min2}) into a more simple form, we must decompose the problem. Because $p_n$ and $B_n$ are in the same term, we thereby separate problem (\ref{equa:min2}) into two subproblems. One has $p_n$ and $B_n$ as optimization variables and the other has $f_n$ and {\color{black}$\mathcal{T}$} as optimization variables.

Two minimization problems are as follows:
\paragraph{Subproblem 1}
\begin{subequations} \label{equa:min5} 
\begin{align}
\min_{f_n, \mathcal{T}} &~ w_1 R_g\sum_{n=1}^{N} \kappa R_lc_nD_nf_n^2+ w_2R_g \mathcal{T} \tag{\ref{equa:min5}} \\
    \text{subject to:} &  ~\text{(\ref{fminfmax}),}~\text{(\ref{constra:time}).}  \notag
\end{align}
\end{subequations}

\paragraph{Subproblem 2}
\begin{subequations} \label{equa:min6}
\begin{align}
\min_{p_n, B_n} &~ w_1R_g \sum_{n=1}^{N} p_n\frac{d_n}{B_n \log_2(1+\frac{p_ng_n}{N_0B_n})} \tag{\ref{equa:min6}} \\
    \text{subject to:} &~ \text{(\ref{pminpmax}),~(\ref{constra:bandwidth}),}~\text{(\ref{constra:time}).} \notag
\end{align}
\end{subequations}

Our strategy is to randomly choose feasible $(\bm{p}, \bm{B})$ at first, and solve Subproblem 1 to get the optimal $\bm{f}$. Then, we solve Subproblem 2 to get the optimal $(\bm{p}, \bm{B})$. The whole process consists of iteratively solving Subproblems 1 and 2 until convergence.

\section{Solution of The joint optimization problem} \label{sec:solution_joint_prob}
In this section, we will explain how to solve Subproblems  1 and 2. 

\subsection{Solution to Subproblem 1} 
It is easy to verify that the function of variables $f_n$ and $\mathcal{T}$ in Eq. (\ref{equa:min5}) is convex. {\color{black} To solve a convex optimization problem, a typical approach is to use the Karush--Kuhn--Tucker (KKT) conditions. For reader's convenience, we briefly describe KKT below while more details can be found in Section 5.5.3 of a book by Boyd~et~al.~\cite{boyd2004convex}. For a convex optimization problem, as long as its objective function and constraints are differentiable, and it satisfies Slater's condition\footnote{In brief, Slater's condition states that the optimal solution must be interior of the feasible region.}, KKT conditions are sufficient and necessary conditions to get the optimal solution.}

Then, we treat the optimization problem by the KKT approach and introduce the Lagrange function as follows:
\begin{align}
    L_1(f_n, \mathcal{T}, \bm{\lambda}) = w_1R_g \sum_{n=1}^N \kappa R_l c_n D_n f_n^2 + w_2R_g  \mathcal{T} \notag\\ + \sum_{n=1}^N \lambda_n [ (\frac{R_lc_nD_n}{f_n}+T_n^{up})-\mathcal{T}],
\end{align}
where $\bm{\lambda}$ is the multiplier associated with the inequality constraint (\ref{constra:time}). After applying KKT conditions, we get 
\begin{align}
    & \frac{\partial L_1}{\partial f_n} \!=\! 2w_1 R_g R_l\kappa c_nD_nf_n \!-\! \lambda_n\frac{R_lc_nD_n}{f_n^2} \!=\! 0 ,~ \forall n \!\in\! \mathcal{N}, \label{KKT:f_n} \\
    & \frac{\partial L_1}{\partial \mathcal{T}} = w_2R_g-\sum_n^N\lambda_n = 0, \label{KKT:T} \\
      & \lambda_n [ (\frac{R_lc_nD_n}{f_n}+T_n^{up})-\mathcal{T}] = 0,
\end{align}
from which we derive 
\begin{align}
    f_n^* = \sqrt[3]{\frac{\lambda_n}{2w_1R_g\kappa}}, \label{relation_f_lamda} \text{ and }
    w_2 = \frac{\sum_n^N \lambda_n}{R_g}.
\end{align}
Then, the dual problem is 
\begin{subequations} \label{equa:subproblem1_dual}
\begin{align} 
\max_{\lambda_n}~&\sum_{n=1}^N (2^{-\frac{2}{3}} + 2^{\frac{1}{3}}) hc_nD_n\lambda_n^{\frac{2}{3}} + T_n^{up}\lambda_n \tag{\ref{equa:subproblem1_dual}} \\
\text{subject to}, \notag \\
&\sum_{n=1}^N \lambda_n =  w_2R_g, \\
& \lambda_n \geq 0,
\end{align}
\end{subequations}
where $h = R_l(w_1\kappa R_g)^{\frac{1}{3}}$. Obviously, this dual problem is a simple convex optimization problem. In this paper, we use CVX \cite{grant2014cvx} 
to solve it and get the optimal $\bm{\lambda^*} = [\lambda_1^*,..., \lambda_n^*]$. Then, we are able to calculate $\bm{f}^*$ through equation (\ref{relation_f_lamda}) and get $\bm{f}$ in the below,
{\color{black} 
\begin{align} \label{relation_f_lamda2}
\bm{f} = \min(\bm{f}^{min}, \max(\bm{f}^*,\bm{f}^{min})).
\end{align}
} 

\subsection{Subproblem 2}
Subproblem 2 is a typical minimization sum-of-ratios problem and thus very hard to solve. Firstly, we define
\begin{align}
G_n(p_n, B_n) := B_n\log_2(1+\frac{p_ng_n}{N_0B_n})
\end{align}
to {\color{black} simplify}
 the notations.
Additionally, (\ref{constra:time}) can be written as $r_n \ge r_n^{min}$, where $r_n^{min} = \frac{d_n}{\mathcal{T}-\frac{R_lc_nD_n}{f_n}}$.
In fact, $G_n(p_n, B_n)$ is the same as the data rate $r_n$ defined in (\ref{def:data_rate}).

Then, we transform Subproblem 2 into the epigraph form through adding an auxiliary variable $\beta_n$ and let $\frac{p_nd_n}{G_n(p_n, B_n)} \le \beta_n$. Then, we get an equivalent problem called \textbf{\textit{SP2\_v1}}.
\begin{subequations} \label{equa:subproblem2_transform}
\begin{align}
    \textbf{\textit{SP2\_v1:}}&~\min_{p_n, B_n, \beta_n}  w_1R_g \sum_{n=1}^{N} \beta_n \tag{\ref{equa:subproblem2_transform}}\\
    \text{subject to}, \notag \\
    &\text{(\ref{pminpmax})},~\text{(\ref{constra:bandwidth}),} \notag \\
    & G_n(p_n, B_n) \ge r_n^{min}, ~\forall n \in \mathcal{N}, \label{constra:subproblem2_data_rate_min}\\ 
    & p_nd_n -\beta_nG_n(p_n, B_n) \le 0,~\forall n \in \mathcal{N}, \label{constra:epi_data_rate} 
\end{align}
\end{subequations}
where $\beta_n$ is an auxiliary variable. Because $\frac{p_nd_n}{G_n(p_n, B_n)} \le \beta_n$ and $G_n(p_n, B_n) > 0$, we can get constraint (\ref{constra:epi_data_rate}).

However, due to constraint (\ref{constra:epi_data_rate}), problem \textbf{\textit{SP2\_v1}} is \mbox{non-convex}, which is still intractable. To continue transforming this problem, we need to give the following lemma.
\begin{lemma} \label{lemma:requisite_fra_prog}
$G_n(p_n, B_n)$ is a concave function.
\end{lemma}
\begin{proof}
See Appendix A.
\end{proof}
The original objective function in Subproblem 2 is the sum of fractional functions. It is obvious that $p_nd_n$ is convex, and Lemma \ref{lemma:requisite_fra_prog} reveals $G_n(p_n, B_n)$ is concave. With such trait, problem \textbf{\textit{SP2\_v1}} can be transformed into a subtractive-form problem via the following theorem.

\begin{theorem} \label{theor:equal_prob_subp2}
If $(\bm{p}^*, \bm{B}^*, \bm{\beta}^*)$ is the solution of problem \textbf{SP2\_v1}, there exists $\bm{\nu}^*$ satisfying that $(\bm{p}^*, \bm{\beta}^*)$ is a solution of the following problem with  $\bm{\nu} = \bm{\nu}^*$ and $\bm{\beta} = \bm{\beta}^*$.
\begin{subequations} \label{equa:subproblem2_lagrangian_min}
\begin{align}
\textbf{SP2\_v2:}&~\min_{p_n, B_n} \sum_{n=1}^N \nu_n(p_nd_n-\beta_nG_n(p_n, B_n)) \tag{\ref{equa:subproblem2_lagrangian_min}} \\
 \text{subject to:}&~ \text{(\ref{pminpmax})},~\text{(\ref{constra:bandwidth}),}~\text{(\ref{constra:subproblem2_data_rate_min}).} \notag
\end{align}
\end{subequations}
Moreover, under the condition of $\bm{\nu} = \bm{\nu}^*$ and $\bm{\beta} = \bm{\beta}^*$, ($\bm{p}^*, \bm{B}^*$) satisfies the following equations.
\begin{align}
    \nu_n^* &= \frac{w_1R_g}{G_n(p_n^*, B_n^*)},~n=1,\cdots,N,  \label{equa:nu_update} \\
   \text{and } \beta_n^* &= \frac{p_n^*d_n}{G_n(p_n^*, B_n^*)}, ~n=1,\cdots, N. \label{equa:beta_update}
\end{align}

\end{theorem}
\begin{proof}
The proof follows  from Lemma 2.1 in \cite{jong2012efficient}.
\end{proof}

\subsection{Solution to Subproblem 2}
\textbf{Theorem \ref{theor:equal_prob_subp2}} indicates that problem \textbf{\textit{SP2\_v2}} has the same optimal solution $(\bm{p^*, \bm{B}^*})$ with problem \textbf{\textit{SP2\_v1}}. Consequently, we can first provide $(\bm{\nu}, \bm{\beta})$ to obtain a solution $(p_n, B_n)$ through solving problem \textbf{\textit{SP2\_v2}}. Then, the second step is to calculate $(\bm{\nu}, \bm{\beta})$ according to (\ref{equa:nu_update}) and (\ref{equa:beta_update}). The first step can be seen as an inner loop to solve problem \textbf{\textit{SP2\_v2}} and the second step is the outer loop to obtain the optimal $( \bm{\nu},\bm{\beta})$. The iterative optimization algorithm is given in Algorithm \ref{algo:MN}, which is a Newton-like method. The original algorithm is provided in \cite{jong2012efficient}. In Algorithm \ref{algo:MN}, we define 
\begin{align}
    \phi_1(\bm{\beta})\!\hspace{-1pt}& =\hspace{-1pt}\! [-p_1d_1 \!\hspace{-1pt}+\hspace{-1pt}\! \beta_1G_1(p_1 \!,\! B_1), \! \cdots \!,\! -p_Nd_N \!\hspace{-1pt}+\hspace{-1pt}\!\beta_NG_N(p_N \!, \!B_N)]^T \!,\\
    \phi_2(\bm{\nu}) \!\hspace{-1pt}& =\hspace{-1pt} \! [-w_1R_g \!\hspace{-1pt} +\hspace{-1pt} \! \nu_1G_1(p_1 \!, \! B_1) \!, \!\cdots \!,\! -w_1R_g \!\hspace{-1pt} +\hspace{-1pt}\! \nu_NG_N(p_N \!, \!B_N)]^T \!,\!
\end{align}
which are two vectors in $\mathbb{R}^{N}$, and
\begin{align}
    & \phi(\bm{\beta}, \bm{\nu}) = [\phi_1^T(\bm{\beta}), \phi_2^T(\bm{\nu})]^T,
\end{align}
which is a vector in $\mathbb{R}^{2N}$. Additionally, the Jacobian matrices of $\phi_1(\bm{\beta})$ and $\phi_2(\bm{\nu})$ are
\begin{align}
   \phi_1^{\prime}(\bm{\beta}) = diag(G_n(p_n, B_n)|_{n = 1,\cdots, N}), \\
      \phi_2^{\prime}(\bm{\nu}) = diag(G_n(p_n, B_n)|_{n = 1,\cdots, N}) ,
\end{align}
where $diag()$ represents a diagonal matrix. \cite{jong2012efficient} proves that when $\phi(\bm{\beta}, \bm{\nu})=\bm{0}$, the optimal solution $(\bm{\nu }, \bm{\beta})$ is achieved. Simultaneously, it satisfies the system of equations (\ref{equa:nu_update}) and (\ref{equa:beta_update}).
Thus, a Newton-like method can be used to update $(\bm{\nu }, \bm{\beta})$, which is done in (\ref{update_nu_beta}).
\begin{algorithm}
\caption{Optimization of $(\bm{p}, \bm{B})$ and $(\bm{\nu}, \bm{\beta})$}
\label{algo:MN}
Initialize $i = 0$, $\xi \in (0, 1)$, $\epsilon \in (0, 1)$. Given feasible ($\bm{p}^{(0)}$, $\bm{B}^{(0)}$). \\
\Repeat{$\phi(\bm{\beta}^{(i)},\bm{\nu}^{(i})) = \bm{0}$ or the number of iterations reaches the maximum iterations $i_0$}{
Calculate 
$$\nu_n^{(i)} = \frac{w_1R_g}{G_n(p_n^{(i)}, B_n^{(i)})},~\beta_n^{(i)} = \frac{p_nd_n}{G_n(p_n^{(i)}, B_n^{(i)})}$$
for $n=1,\cdots, N$.

Obtain $(\bm{p}^{(i+1)}$, $\bm{B}^{(i+1)})$ through solving \textbf{\textit{SP2\_v2}} according to \textbf{Theorem \ref{theor:express_B_p}} by using CVX given $(\bm{\nu}^{(i)}, \bm{\beta}^{(i)})$.

Let $j$ be the smallest integer which satisfies
\begin{align} \label{newton_method}
    | \phi(\bm{\beta}^{(i)}+\xi^j\bm{\sigma_1}^{(i)}, \bm{\nu}^{(i)}+\xi^j\bm{\sigma_2}^{(i)}) | \notag \\ \le (1-\epsilon\xi^j) | \phi(\bm{\beta}^{(i)}, \bm{\nu}^{(i)}) | ,
\end{align}
where 
\begin{align}
    \bm{\sigma_1}^{(i)} = -[\phi_1^{\prime}(\bm{\beta}^{(i)})]^{-1}\phi_1(\bm{\beta}^{(i)}),\notag\\ \bm{\sigma_2}^{(i)} = -[\phi_2^{\prime}(\bm{\nu}^{(i)})]^{-1}\phi_2(\bm{\nu}^{(i)}).
\end{align}

Update 
\begin{align} \label{update_nu_beta}
    (\bm{\beta}^{(i+1)},\bm{\nu}^{(i+1)}) =( \bm{\beta}^{(i)} + \xi^j\bm{\sigma_1}^{(i)},\bm{\nu}^{(i)}+\xi^{j}\bm{\sigma_2}^{(i)}).
\end{align}

Let $i \leftarrow i+1$.
}
\end{algorithm}

Until now, the rest problem we need to focus on is problem \textbf{\textit{SP2\_v2}}. Remember that $\bm{\nu}$ and $\bm{\beta}$ are fixed in this problem, which can be seen as constants. Given Lemma \ref{lemma:requisite_fra_prog}, $p_nd_n$ is convex, and three constraints \text{(\ref{pminpmax})}~\text{(\ref{constra:bandwidth})}~\text{(\ref{constra:subproblem2_data_rate_min})} are convex, the objective function (\ref{equa:subproblem2_lagrangian_min}) in problem \textbf{\textit{SP2\_v2}} is convex. Thus, KKT conditions are the sufficient and necessary optimality conditions for finding the optimal solution.


Before applying KKT conditions, we write down the partial Lagrangian function of problem \textbf{\textit{SP2\_v2}}: \\
$L_2(p_n, B_n, \tau_n, \mu)=$
\begin{align} \label{equal:sp2_v2_lagran}
    &\sum_{n=1}^N \nu_n(p_nd_n-\beta_nB_n\log_2(1+\frac{p_ng_n}{N_0B_n}))\notag \\& - \sum_{n=1}^N \tau_n(B_n\log_2(1 +\frac{p_ng_n}{N_0B_n})-r_n^{min}) + \mu(\sum_{n=1}^NB_n - B),
\end{align}
where $\tau_n|_{n=1,\ldots,N}$ and $\mu$ are non-negative Lagrangian multipliers.


After applying KKT conditions to problem \textbf{\textit{SP2\_v2}}, we get
\begin{align}
    \frac{\partial L_2}{\partial p_n} & = \nu_n(d_n-\frac{\beta_ng_n}{N_0(1+\vartheta_n)\ln2})-\frac{\tau_ng_n}{N_0(1+\vartheta_n)\ln2} =0, \label{lag:partial_p}\\
    \frac{\partial L_2}{\partial B_n} &= -(\nu_n\beta_n+\tau_n)\log_2(1+\vartheta_n)+\frac{(\nu_n\beta_n+\tau_n)p_ng_n}{(1+\vartheta_n)\ln2 N_0B_n}\notag \\& +\mu = 0,  \label{lag:partial_B}\\
    -\tau_n & (B_n\log_2(1+\vartheta_n)-r_n^{min}) = 0, \label{lag:data_rate}\\
    \mu (\sum_{n=1}^N & B_n-B) = 0, \label{lag:mu}
\end{align}
where $\vartheta_n = \frac{p_ng_n}{N_0B_n}$ for $n=1,\cdots,N$.

With (\ref{lag:partial_p})--(\ref{lag:mu}), \textbf{Theorem \ref{theor:express_B_p}} is given to find the optimal bandwidth and transmission power for each device. 
\begin{theorem} \label{theor:express_B_p}
The optimal bandwidth $\bm{B}$ and transmission power $\bm{p}$ are expressed as 
\begin{align}
   & B_n^* = 
   \begin{cases} 
   & \frac{r_n^{min}}{\log_2(1+\Lambda_n)},~\text{if}~\tau_n \neq 0,~ n = 1,\cdots, N, \\
   & \text{Solution~to~problem (\ref{SP2_v3})} ,~\text{if}~\tau_n = 0,
   \end{cases} \label{express_B}\\
   & p_n^* = \min(p_n^{max},\max(\Gamma(B_n),~p_n^{min})) \label{express_p}\notag\\&\quad\quad\text{when}~\mu \neq 0,~ n=1,\cdots, N,
\end{align}
where 
\begin{align}
 \Lambda_n &= \frac{(\nu_n\beta_n+\tau_n)g_n}{N_0d_n\nu_n \ln2},\\
 \Gamma(B_n) &= (\frac{(\nu_n\beta_n+\tau_n)g_n}{N_0d_n\nu_n\ln2}-1)\frac{N_0B_n}{g_n}.
\end{align}

\end{theorem}
\begin{proof}
See Appendix B.
\end{proof}
Then, we have solved problem \textbf{\textit{SP2\_v2}}.

\subsection{Resource Allocation Algorithm} \label{sec6}
Until now, we have already solved Subproblem 1 and Subproblem 2. Here we give the complete resource allocation algorithm in Algorithm \ref{algo:resource_allocation_algorithm}.

\begin{algorithm} 
\caption{Resource Allocation Algorithm}
\label{algo:resource_allocation_algorithm}
Initialize $ sol^{(0)} = (\bm{p}^{(0)}, \bm{B}^{(0)}, \bm{f}^{(0)})$ of problem (\ref{equa:min2}), iteration number $k=1$.\\
\Repeat{$|sol^{(k)}-sol^{(k-1)}| \leq \epsilon_0$ or the number of iterations achieves $K$}{
Solve Subproblem 1 through solving problem (\ref{equa:subproblem1_dual}) with CVX given $(\bm{p}^{(k-1)}, \bm{B}^{(k-1)})$. Obtain $\bm{f}^{(k)}$ according to equation (\ref{relation_f_lamda2}).\\
Solve Subproblem 2 through Algorithm \ref{algo:MN} and obtain $(\bm{p}^{(k)}, \bm{B}^{(k)})$ with CVX.\\
$ sol^{(k)} = (\bm{p}^{(k)}, \bm{B}^{(k)}, \bm{f}^{(k)})$.\\
Set $k \leftarrow k+1$.
}
\end{algorithm}
Algorithm \ref{algo:resource_allocation_algorithm} first initializes a feasible solution according to the range of power $\bm{p}$ and bandwidth $\bm{B}$ and the sum of the bandwidth of each device cannot exceed $B$. 
Then, solving problem (\ref{equa:subproblem1_dual}) with given $(\bm{p}, \bm{B})$ can make sure we get an optimal $\bm{f}$ at this step. 
Subproblem 2 does not include $\bm{f}$ and it is an optimization problem which focuses on $\bm{p}$ and $\bm{B}$.

\section{Convergence and Time Complexity} \label{sec:conv}
In this section, we analyze the convergence of our resource allocation algorithm (Algorithm \ref{algo:resource_allocation_algorithm}) and present the time complexity of Algorithm \ref{algo:resource_allocation_algorithm}.

\subsection{Convergence Analysis}
In Algorithm \ref{algo:resource_allocation_algorithm}, given $(\bm{p}^{(k-1)}, \bm{B}^{(k-1)})$ at iteration $k$, we can obtain an optimal solution $\bm{f}^{(k)}$ of problem (\ref{equa:subproblem1_dual}).

Then, Algorithm \ref{algo:MN} is called in Algorithm \ref{algo:resource_allocation_algorithm}, so next we discuss the convergence of Algorithm \ref{algo:MN}. To simplify the notation, we define $\bm{\alpha} := (\bm{\beta}, \bm{\nu})$.

According to Theorem 3.2 in \cite{jong2012efficient}, if the following three conditions hold: (i) $\phi(\bm{\alpha})$ is differentiable and satisfies the Lipschitz condition in a solution set $\Omega$, (ii) assume a number $L > 0$ such that for $\bm{\alpha_1}, \bm{\alpha_2}$ $\in \Omega$,
\begin{align} \label{M_condition1}
    || \phi^\prime(\bm{\alpha_1}) - \phi^\prime(\bm{\alpha_2})|| \le L ||\bm{\alpha_1}-\bm{\alpha_2}||,
\end{align}
and (iii) there exits a number $M > 0$ so that for each $\bm{\alpha} \in \Omega$,
\begin{align} \label{M_condition}
|| [\phi^\prime(\bm{\alpha})]^{-1} || \le M,
\end{align}
then no matter at which starting point $\bm{\alpha}^0 \in \Omega$, (\ref{newton_method}) and (\ref{update_nu_beta}) in Algorithm \ref{algo:MN} will converge to a unique solution of $\phi(\bm{\alpha}) = 0$ in a linear rate. In addition, the convergence rate in the solution's neighborhood is quadratic.

$\phi(\bm{\alpha})$ in Algorithm \ref{algo:MN} is actually linear and it is easy to prove it satisfies the Lipshitz condition and the conditions (\ref{M_condition1}) (\ref{M_condition}). Therefore, Algorithm \ref{algo:MN} converges with a linear rate at any starting point $\bm{\alpha}^0 \in \Omega$ and a quadratic convergence rate of the solution's neighborhood.

Therefore, iteratively solving Subproblem 1 to get $\bm{f}$ and Subproblem 2 to get ($\bm{p}, \bm{B}$) will converge eventually.

\subsection{Time Complexity}
For Algorithm \ref{algo:MN}, we can use floating point operations (flops) to count the time complexity. One real addition/multiplication/division is counted as one flop. We mainly focus on analyzing Steps 3--6 because they take a major part of computational complexity. Step 3 needs $2N$ flops. Step 4 utilizes CVX to solve problem \textit{\textbf{SP2\_v2}}. Besides, CVX invokes an interior-point algorithm to solve Semidefinite Programming (SDP) problems. Thus, given the solution accuracy $\epsilon_1 > 0$, the worst time complexity is $\mathcal{O}(N^{4.5}\log\frac{1}{\epsilon_1})$ \cite{luo2010semidefinite}. Additionally, 
Step 5 will take $\mathcal{O}((j+1)N)$ flops. In fact, $j$ is a small number usually. If $j=0$ (i.e., $\xi^j=1$), Step 6 is just the Newton method. Otherwise, it is a Newton-like method. Therefore, the convergence rate is superlinear or quadratic \cite{jong2012efficient}. Obviously, Step 6 takes $4N$ flops. Since the number of iterations is not larger than $i_0$, the time complexity of Algorithm \ref{algo:MN} will not exceed $\mathcal{O}(i_0(N^{4.5}\log\frac{1}{\epsilon_1}+(j+1)N))$.

As for Algorithm \ref{algo:resource_allocation_algorithm}, it solves two subproblems successively through using the convex package CVX in each iteration. Steps 3 and 4 both use CVX to solve two subproblems. As we have mentioned above, the time complexity caused by CVX is $\mathcal{O}(N^{4.5}\log\frac{1}{\epsilon_1})$ given a solution accuracy $\epsilon_1>0$. Step 4 invokes Algorithm \ref{algo:MN}, and we have already analyzed the time complexity of Algorithm \ref{algo:MN}. With the aforementioned time complexity of Algorithm \ref{algo:MN}, the overall complexity is $\mathcal{O}(K(i_0+1)N^{4.5}\log\frac{1}{\epsilon_1}+Ki_0(j+1)N)$.

\section{Experiments} \label{sec:experiments}

We will conduct extensive experiments to illustrate the effectiveness of our proposed algorithm. We first present the parameter setting of our experiments in Section~\ref{sec-exp-Parameter}, and then report various experimental results in other subsections.

\subsection{Parameter Setting} \label{sec-exp-Parameter}
For most experimental settings, we use the settings used in \cite{yang2021energy}. In the experiments, $N$ denoting the number of devices is 50. The devices are uniformly located in a circular area of size $500$ m $\times$ $500$ m and the center is a base station. The channel's pass loss is modeled as $128.1 + 37.6\log(\texttt{distance})$ along with the standard deviation of shadow fading, which is $8$ dB, and the unit of \texttt{distance} is kilometer. The power spectral density of Gaussian noise $N_0$ is $-174$ dBm/Hz.

Additionally, the default number of local iterations $R_l$ is set as 10 and the number of global aggregations $R_g$ is 400. 
The data size for each device to upload is set as $28.1$ kbits.  
Besides, each device has $500$ samples. The number of CPU cycles per sample, $c_n$, is chosen randomly from $[1, 3]\times 10^4$.
The effective switched capacitance $\kappa$ is $10^{-28}$.
The maximum frequency $f_n^{max}$ and the maximum transmission power $p_n^{max}$ for all devices are $2$ GHz and $12$ dBm, respectively. The minimum transmission power $p_n^{min}$ is set as $0$ dBm. The total bandwidth $B$ is $20$ MHz.

\subsection{How does Our Optimization Algorithm Perform with respect to System Parameters?} \label{subsec:adj_weight_param}

{\color{black} In this section, we study the effect of system parameters on our optimization problem. The system parameters include: weight parameters, the number of devices, global rounds \& local iterations, and the radius of the circular area for generating different devices. }

\begin{figure}
\centering
\begin{subfigure}{.25\textwidth}
  \centering
  \includegraphics[width=1\linewidth]{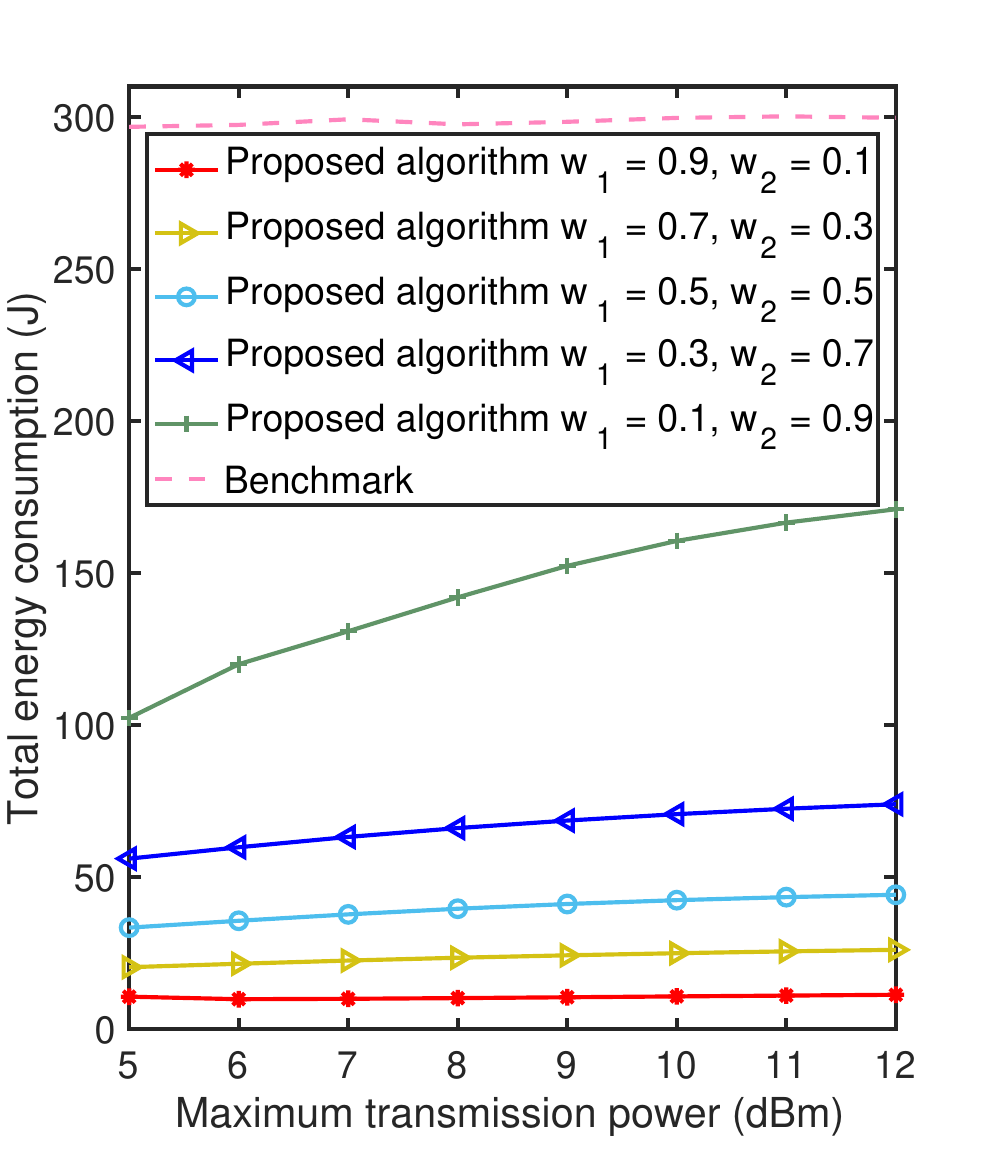}\vspace{-8pt}
    \caption{Total energy consumption}
     \label{fig:total_e_different_p}
\end{subfigure} \hspace{-15pt}
\begin{subfigure}{.25\textwidth}
  \centering
  \includegraphics[width=1\linewidth]{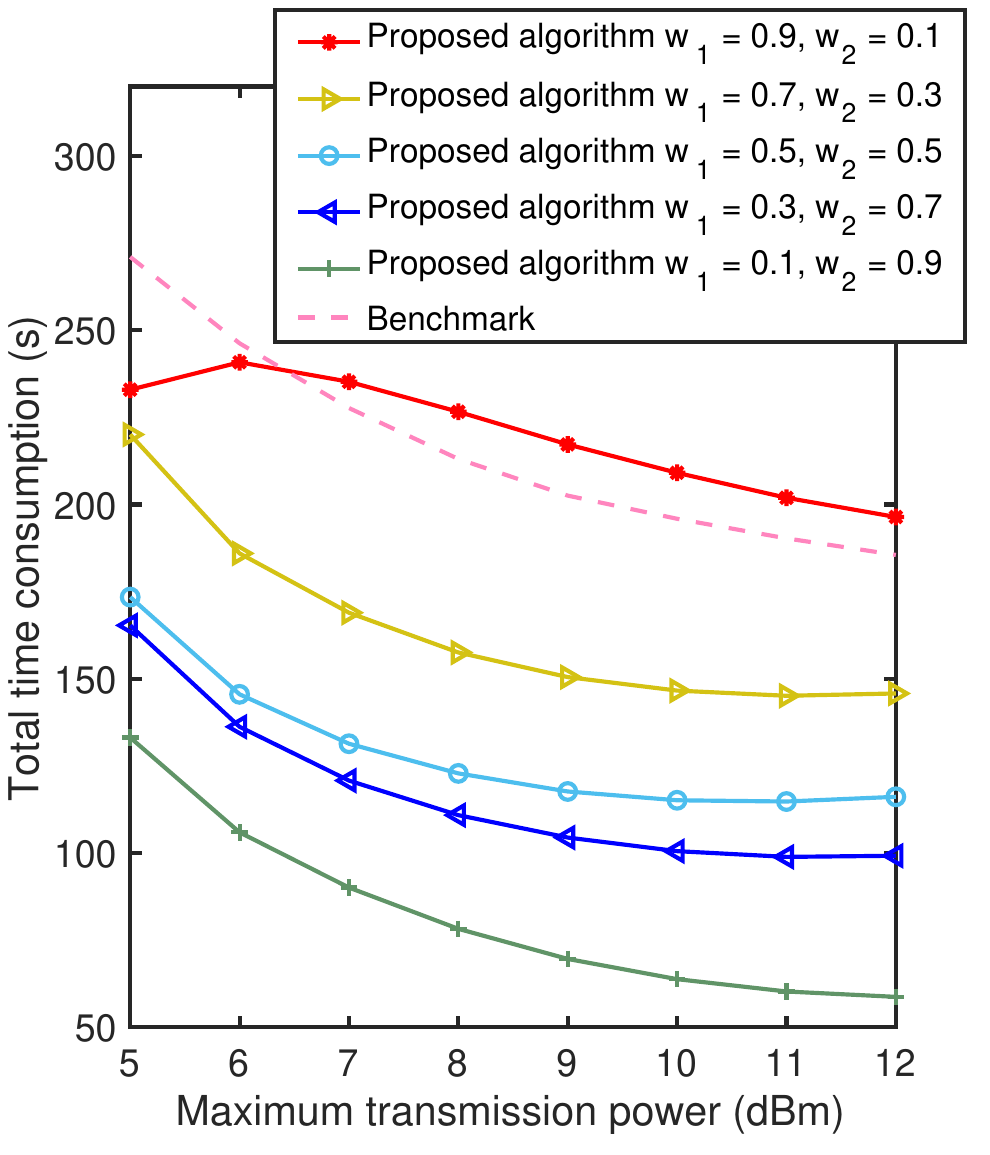}\vspace{-8pt}
    \caption{Total delay\vspace{-8pt}}
    \label{fig:total_time_different_p}
\end{subfigure}\vspace{-5pt}
\caption{Experiments varying maximum transmit power limits.\vspace{-10pt}}
\end{figure}

{\color{black}\textbf{Weight parameters}.} In our optimization problem, we have two weight parameters $w_1$ and $w_2$. The sum of them always equals to $1$. If $w_1$ is greater than $w_2$, energy consumption will take a more dominant position in our optimization problem. Analogously, if $w_2$ is larger than $w_1$, it is more important to optimize time instead of energy.

 In Fig. \ref{fig:total_e_different_p}, Fig. \ref{fig:total_time_different_p}, Fig. \ref{fig:total_e_different_f} and Fig. \ref{fig:total_time_different_f}, we compare five pairs of $(w_1, w_2) = (0.9, 0.1), (0.7, 0.3), (0.5, 0.5), (0.3, 0.7), (0.1, 0.9)$ with a benchmark algorithm. There are six lines in those figures. $(w_1, w_2) = (0.9, 0.1)$ and $(0.7, 0.3)$ are to simulate devices under a low-battery condition, so the weight parameter $w_1$ of the energy consumption is larger than the weight parameter $w_2$ of the time consumption. $(w_1, w_2) = (0.5, 0.5)$ represents the normal scenario, which means no preference for the time or energy. Besides, $(w_1, w_2) = (0.1, 0.9)$ and $(0.3, 0.7)$ stand for the time-sensitive scenario, so $w_2$ is larger than $w_1$. In brief, the algorithms in our comparison include:
\begin{itemize}
    \item Our proposed algorithm with weight parameters $(w_1 = 0.9, w_2 = 0.1)$, $(w_1 = 0.7, w_2 = 0.3)$, $(w_1 = 0.5, w_2 = 0.5)$, $(w_1 = 0.3, w_2 = 0.7)$ and $(w_1 = 0.1, w_2 = 0.9)$.
    \item The benchmark algorithm: 1) When comparing energy consumption and completion time at different maximum transmission power limits, for the $n$-th device, randomly select the CPU frequency $f_n$ from 0.1 to 2 GHz and set $p_n = p^{max}$, $B_n = \frac{B}{N}$. 2) When comparing at different maximum CPU frequencies, randomly select the transmission power $p_n$ between 0 and 12 dBm and set $f_n = f^{max}$, $B_n = \frac{B}{N}$.
\end{itemize}
To be more general, we run our algorithm at each pair of weight parameters 100 times and take the average value of energy and time consumption. At each time, we will generate 50 users randomly. {\color{black}Readers may wonder why the simulation plots do not show confidence intervals.  This is because at each time, we only generate users randomly. For the same group of users, our algorithm will produce the same result for each operation if we do not change any system parameter. We only want to observe the overall performance of our optimization algorithm at different pairs of weight parameters $(w_1, w_2)$.}

Fig. \ref{fig:total_e_different_p} and Fig. \ref{fig:total_time_different_p} show the results of total energy consumption and total time consumption with five pairs of $(w_1, w_2)$ at different maximum transmit power limits, respectively.
It can be observed clearly that when $w_1$ becomes larger and $w_2$ decreases, the total energy consumption becomes smaller and the total time becomes larger. 
The reason is that if $w_1$ (resp., $w_2$) increases, our optimization algorithm will pay more attention to minimizing the energy cost (resp., time consumption). 

Additionally, in Fig. \ref{fig:total_e_different_p}, 
the green, dark blue, light blue, yellow and red lines are all below the benchmark line with a wide gap, which means our algorithm obtains better total energy consumption even at $w_1 = 0.1$ (this means the algorithm mainly focuses on minimizing the total time consumption instead of energy).

Correspondingly, Fig. \ref{fig:total_time_different_p} reveals that our time optimization is still better than the benchmark except for $w_1=0.9$ and $w_2 = 0.1$ after $p^{max} = 7$ dBm. We analyze that it is caused by finding a better solution to minimize the total energy consumption through sacrificing the total time cost.

\begin{figure}
\centering
\begin{subfigure}{.25\textwidth}
  \centering
  \includegraphics[width=1\linewidth]{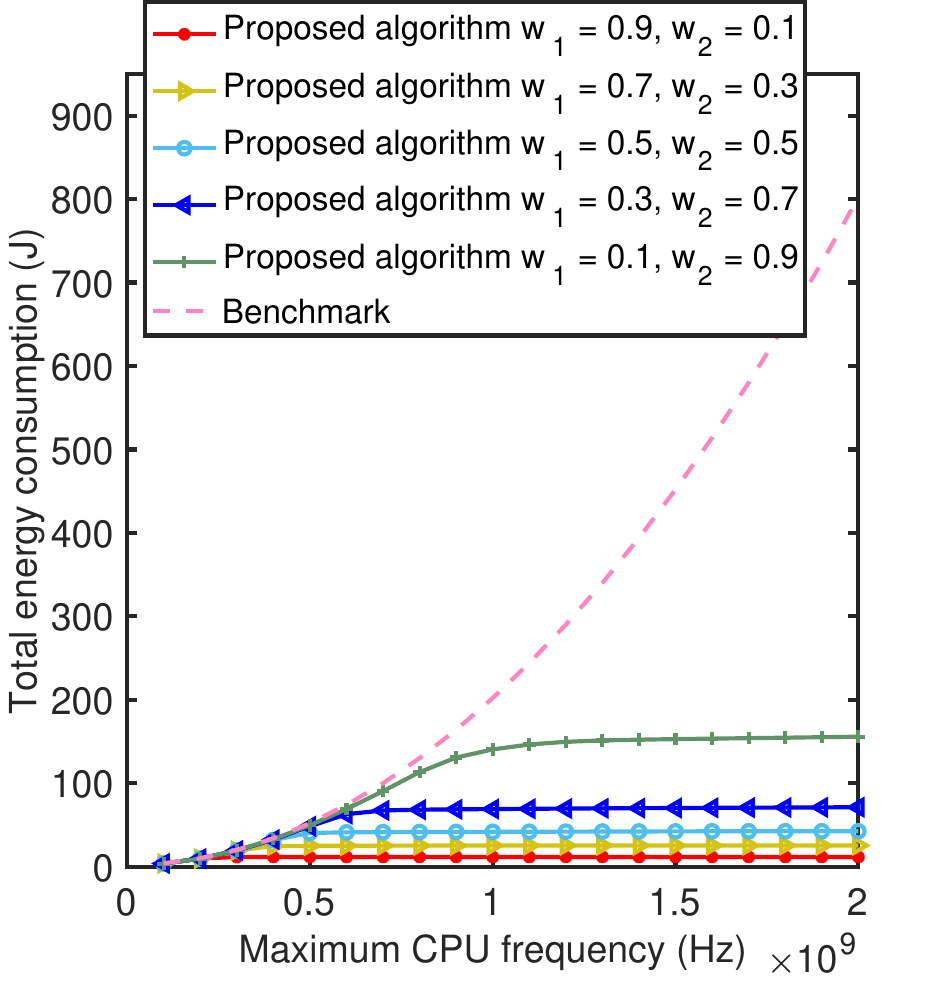}\vspace{-8pt}
    \caption{Total energy consumption}
    \label{fig:total_e_different_f}
\end{subfigure} \hspace{-15pt}
\begin{subfigure}{.25\textwidth}
  \centering
  \includegraphics[width=1\linewidth]{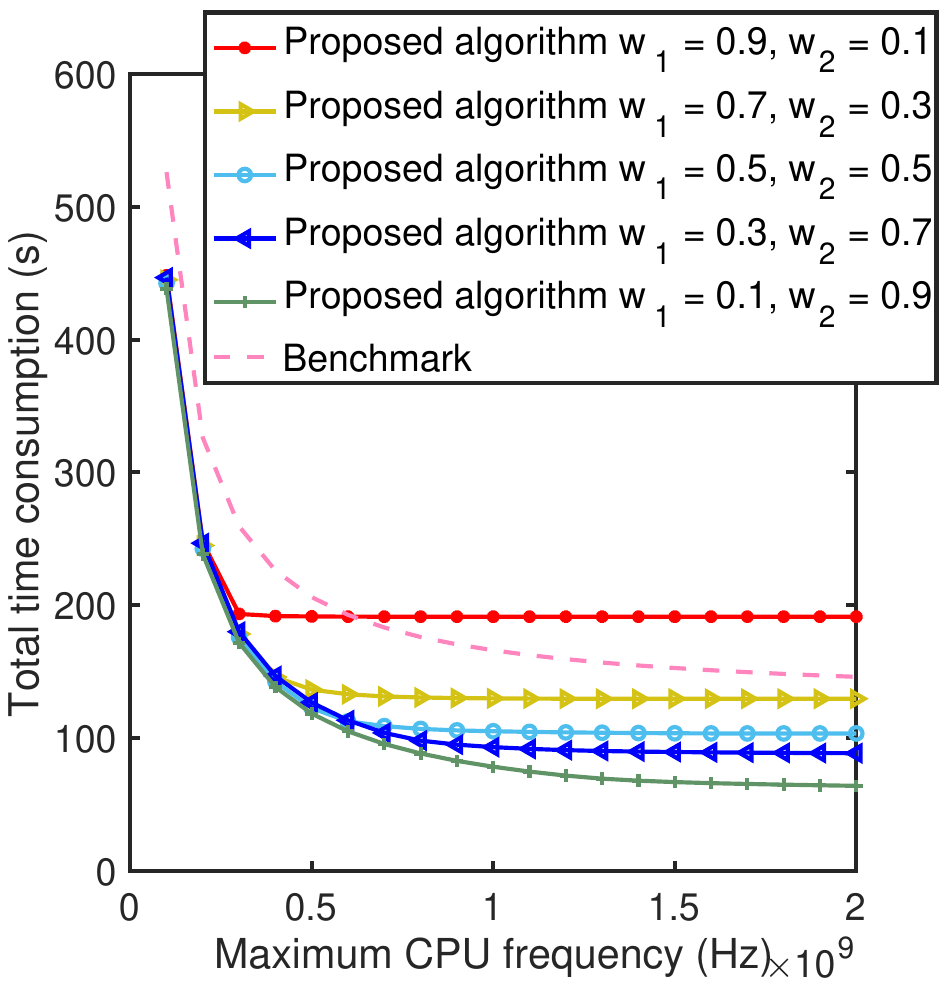}\vspace{-8pt}
    \caption{Total delay}\vspace{-8pt}
    \label{fig:total_time_different_f}
\end{subfigure}\vspace{-3pt}
\caption{Experiments with different maximum CPU frequencies.\vspace{-10pt}}
\end{figure}

    

    

Fig. \ref{fig:total_e_different_f} and Fig. \ref{fig:total_time_different_f} compare the total energy consumption and delay at different maximum CPU frequencies. Obviously, with the increase of the maximum CPU frequency, energy consumption of the benchmark algorithm is rising, which can be seen in Fig. \ref{fig:total_e_different_f}. Intuitively, if the CPU frequency is set at the maximum, the time consumption will be reduced. Therefore, in Fig. \ref{fig:total_time_different_f}, the total time consumption of the benchmark algorithm is decreasing as CPU frequency increases. However, it only performs better than our algorithm when $w_1=0.9$ and $w_2 = 0.1$ after the maximum CPU frequency $=0.6$ GHz. The reason is our algorithm mainly focuses on optimizing the energy when $w_1 = 0.9$ and $w_2 = 0.1$, so our algorithm's energy consumption is the smallest in Fig. \ref{fig:total_e_different_f}.

Besides, five lines related to our algorithm in Fig. \ref{fig:total_e_different_f} and Fig. \ref{fig:total_time_different_f} all enter a stationary phase after a certain maximum CPU frequency. This is because we have already found the optimal frequency for each device at that weight parameters. Thus, there will not be any change with the increasing $f^{max}$.

{\color{black} \textbf{The number of devices}. We also explore the effect of the number of devices on the total energy consumption and the time consumption.
We assume the total number of samples is 25000 and distribute all the data equally to each device.

It can be observed in Fig.~\ref{fig:total_e_different_devices} that with the increase of the number of devices, the total energy consumption decreases. This is caused by the decreasing number of samples allocated to each device as the number of devices increases, which saves the computation energy. In Fig.~\ref{fig:total_time_different_devices}, the overall trend is decreasing. As we have mentioned, because we increase the number of devices, the number of samples assigned to each device decreases, and the computation time cost decreases. While, for the green line at the bottom in Fig.~\ref{fig:total_time_different_devices}, the total delay increases slightly as the number of devices increases. Since our algorithm optimizes a weighted sum of the energy and delay, sometimes the increase in the delay is to achieve lower energy consumption.
}

\begin{figure}
\centering
\hspace{-8pt}\begin{subfigure}{.25\textwidth}
  \centering
  \includegraphics[width=1\linewidth]{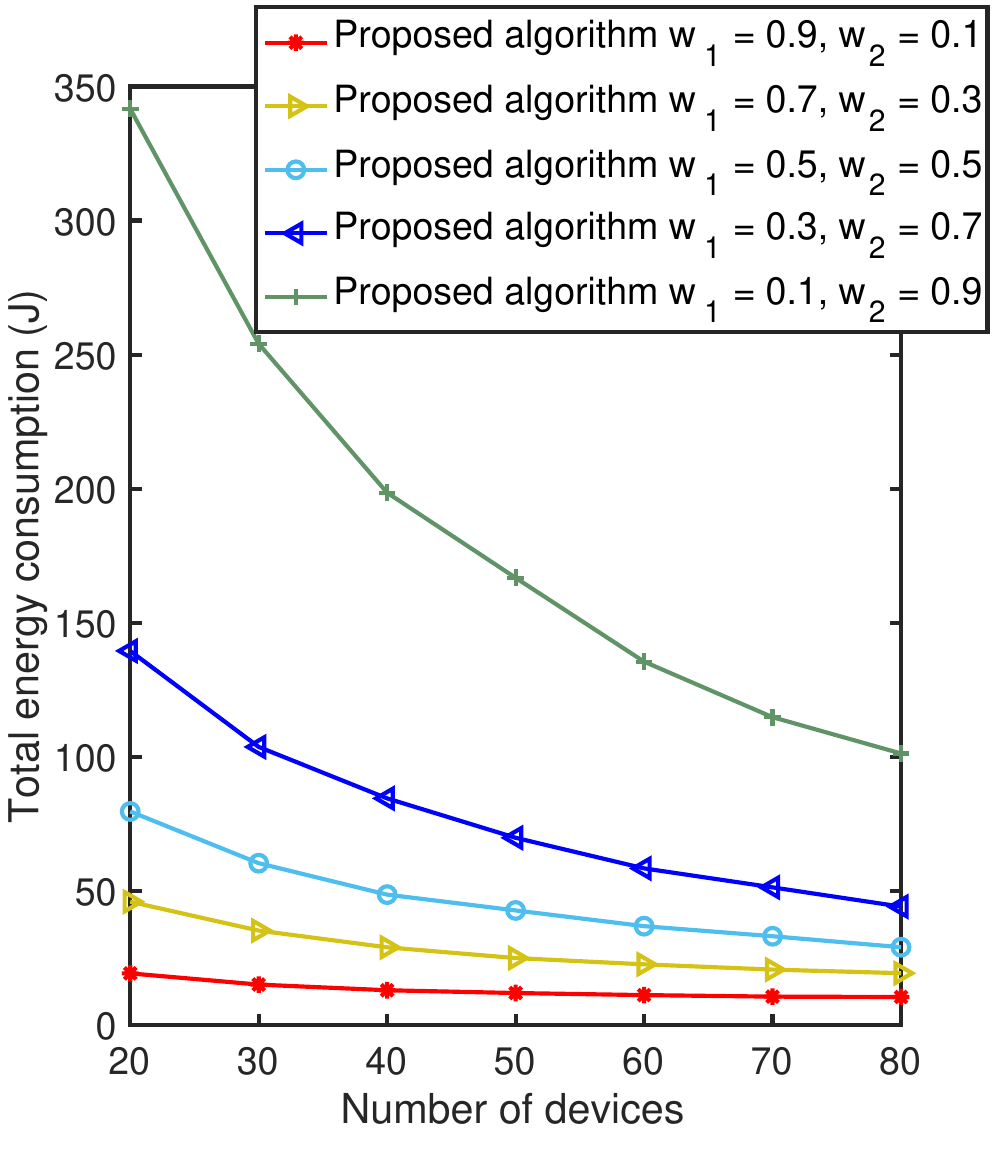}\vspace{-8pt}
    \caption{Total energy consumption}
    \label{fig:total_e_different_devices}
\end{subfigure} \hspace{-8pt}
\begin{subfigure}{.25\textwidth}
  \centering
  \includegraphics[width=1\linewidth]{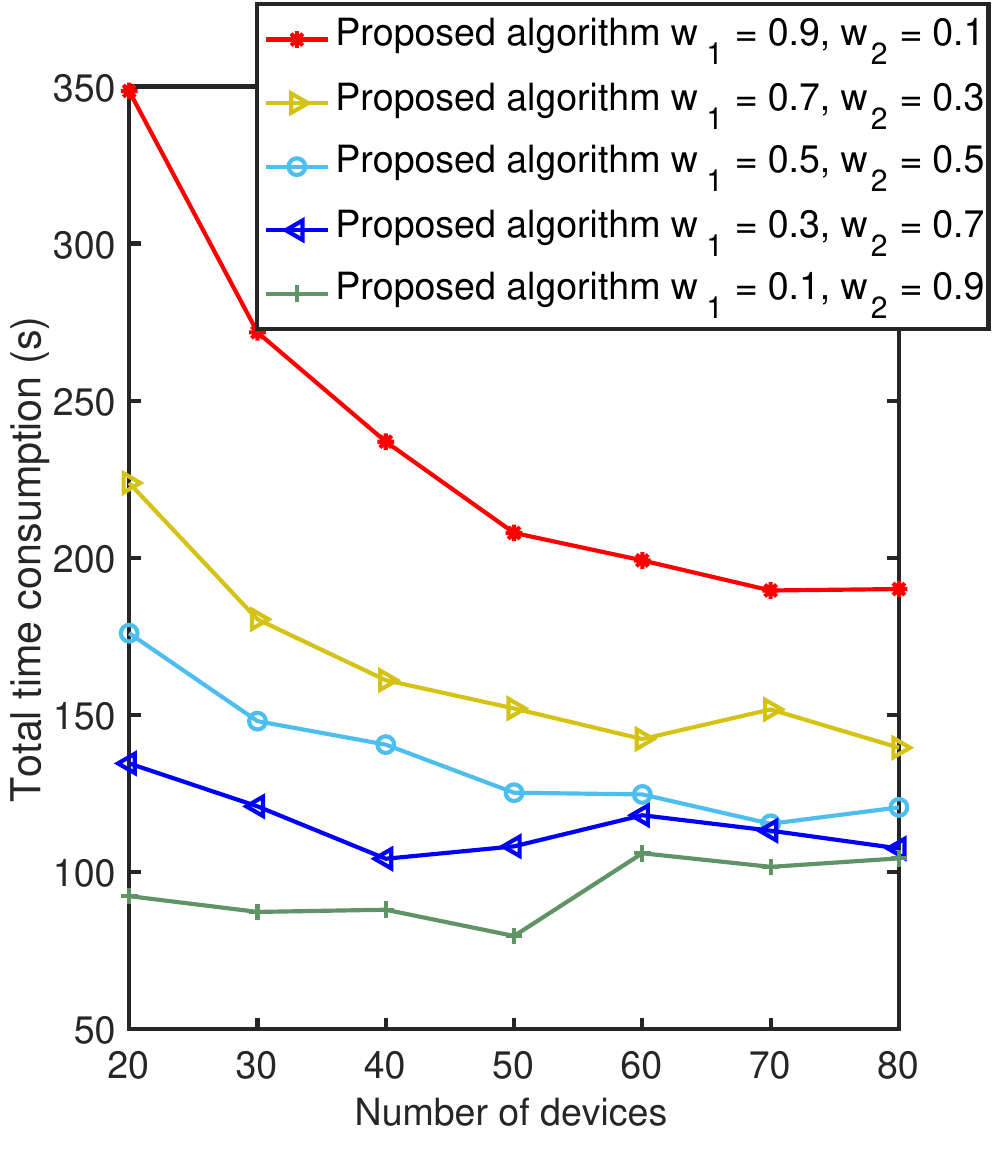}\vspace{-8pt}
    \caption{Total delay\vspace{-8pt}}
    \label{fig:total_time_different_devices}
\end{subfigure}\vspace{-10pt}
\caption{Experiments with different numbers of devices.\vspace{-20pt}}
\end{figure}



{\color{black} \textbf{The radius of the circular area}. 
Fig.~\ref{fig:total_e_different_cir} and Fig.~\ref{fig:total_time_different_cir} show the change in energy and time consumption as the area used for generating users randomly increases. There are three lines in each figure representing  20, 50, 80 devices, respectively.
Fig.~\ref{fig:total_e_different_cir} does not reveal a clear correlation between energy consumption and the radius. Since the number of samples of each device is always 500 in this setting, as the number of devices increases, the overall computation energy consumption must increase unless the CPU frequency is set lower.
Whereas, if the CPU frequency/transmission power is set lower, the computation time/transmission time cost becomes larger. This explains why the total energy consumption of these three lines suddenly drops at $0.9$ km. Since the CPU frequency/transmission power is set lower, the total energy consumption decreases at $0.9$ km.
Fig.~\ref{fig:total_time_different_cir} clearly shows the time consumption and the radius are positively correlated. Intuitively, the transmission time will increase as the radius increases if the transmission power is kept constant. 
To save the time consumption, our optimization algorithm cannot always put the energy consumption first.
Therefore, at $1.3$ km, the CPU frequency and transmission power much be set larger than the previous settings so that the time consumption will not rise much higher. 
}

\begin{figure}
\centering
\begin{subfigure}{.255\textwidth}
  \centering
  \includegraphics[width=1\linewidth]{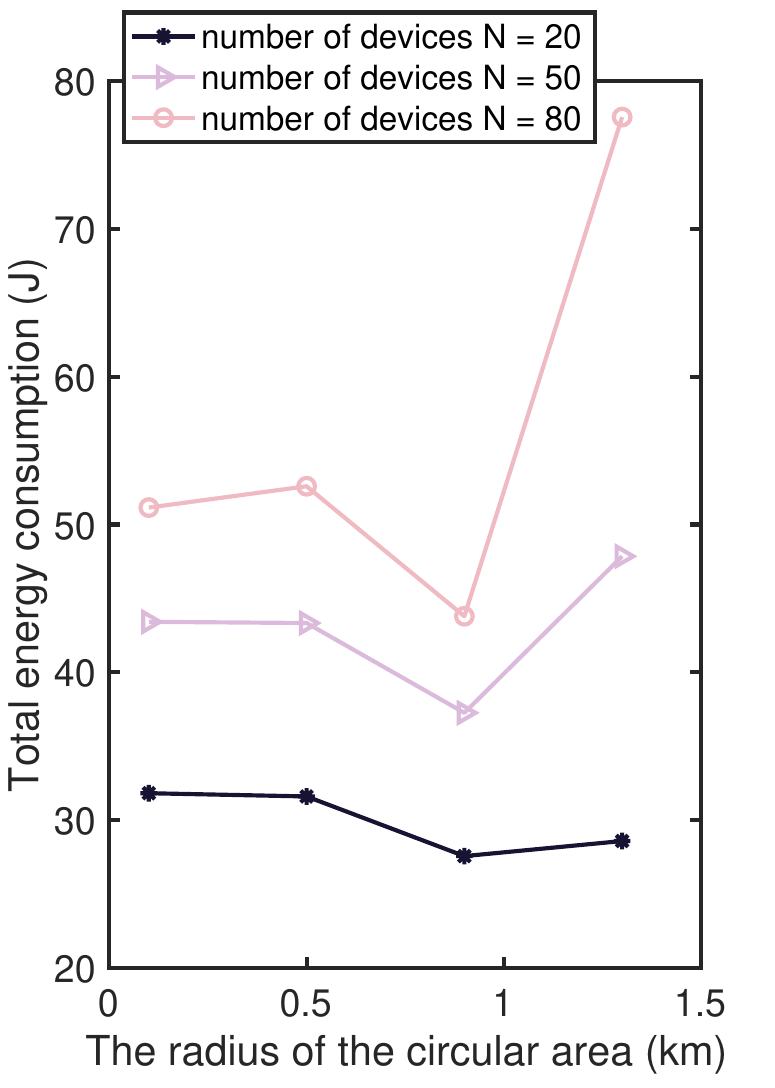}\vspace{-8pt}
    \caption{Total energy consumption}
    \label{fig:total_e_different_cir}
\end{subfigure} \hspace{-15pt}
\begin{subfigure}{.245\textwidth}
  \centering
  \includegraphics[width=1\linewidth]{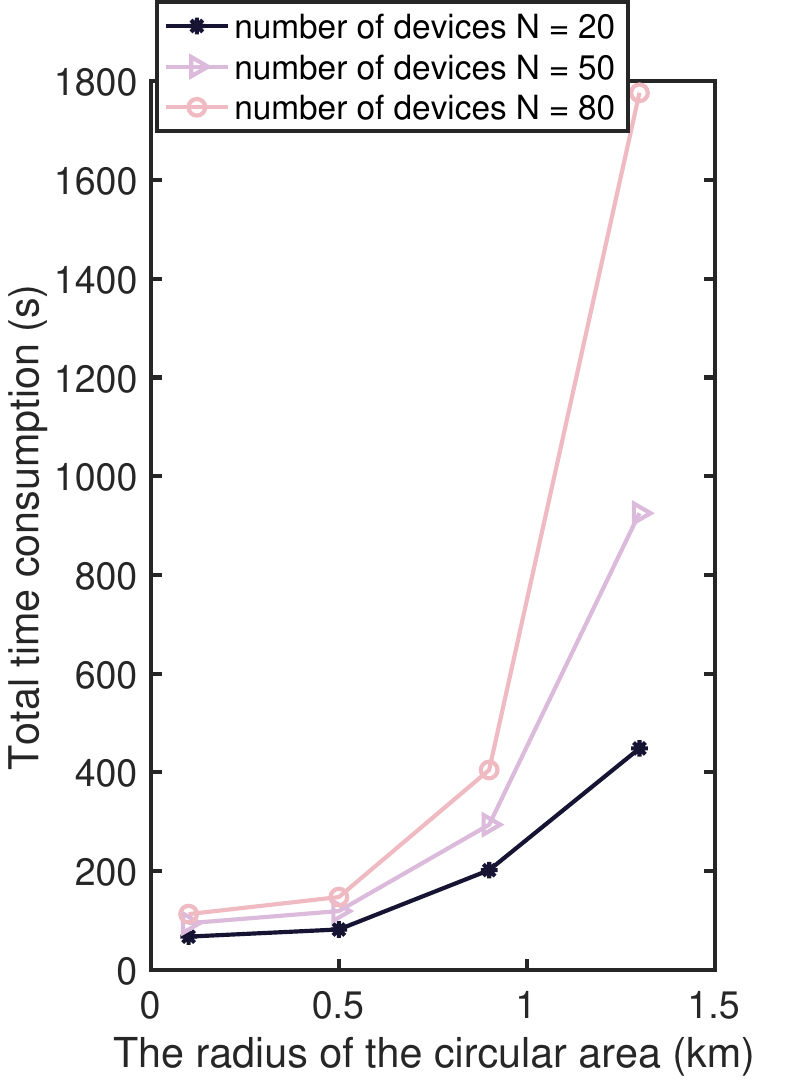}\vspace{-8pt}
    \caption{Total delay\vspace{-8pt}}
    \label{fig:total_time_different_cir}
\end{subfigure}\vspace{-5pt}
\caption{Experiments with different radii of the circular area. $w_1 = w_2 = 0.5$.\vspace{-20pt}}
\end{figure}



{\color{black} \textbf{The communication/computation rounds}.
We choose five kinds of global communication rounds $R_g$: 50, 100, 200, 300, 400. Fig.~\ref{fig:total_e_different_iter} and Fig.~\ref{fig:total_time_different_iter} show how the energy consumption and time cost change with the increase in the number of local iterations $R_l$ per global round. It can be seen that as $R_l$ and $R_g$ become larger, the total energy and time consumption both increase. They are positively correlated. 

}
\textbf{The number of samples on each device}. If we keep all the system parameters consistent and just vary the number of samples on each device, we find that $D_n$ is positively correlated with time and energy. This result is consistent with the analytical expression of the optimization objective. Due to the space limitation, we do not provide experiments where different devices have various numbers of samples.



\subsection{Joint Communication \& Computation Optimization vs. Communication/Computation Optimization Only}
To study the domination relationship between the transmission energy and computation energy, we compare our algorithm with the following schemes at the aspect of total energy consumption at different maximum completion time limits. The maximum transmission power is $p^{max} = 10$ dBm. {\color{black}We denote the maximum delay of the whole training process as $T$.} $T$ is fixed to compare these three schemes, and we set $w_1=1$ and $w_2 = 0$ in our algorithm.
\begin{itemize}
    \item \textbf{Communication optimization only}: Each device's computation frequency is set as a fixed value. 
    We optimize only the transmission power and bandwidth allocated to each device. To guarantee there is a feasible solution, we set the fixed frequency value for each device as $\frac{R_gR_lc_nD_n}{\mathcal{T}-R_g\max(d_n/r_n)}$, which is derived from constraint (\ref{constra:time}), and $r_n$ is calculated from the initial bandwidth and transmission power.
    \item \textbf{Computation optimization only}: Each device's transmission power and bandwidth are fixed and we optimize only  the CPU frequency. The transmission power and bandwidth of device $n$ are set as $p_n = p^{max}$ and $B_n = \frac{B}{2N}$. Such setting gives better simulation results than letting $p_n $ be $ \frac{p^{max}}{2} $ and/or $B_n $ be $ \frac{B}{N}$, and is also used in the source code of \cite{yang2021energy}. 
\end{itemize}

\begin{figure}
\centering
\begin{subfigure}{.25\textwidth}
 \centering
  \includegraphics[width=1\linewidth]{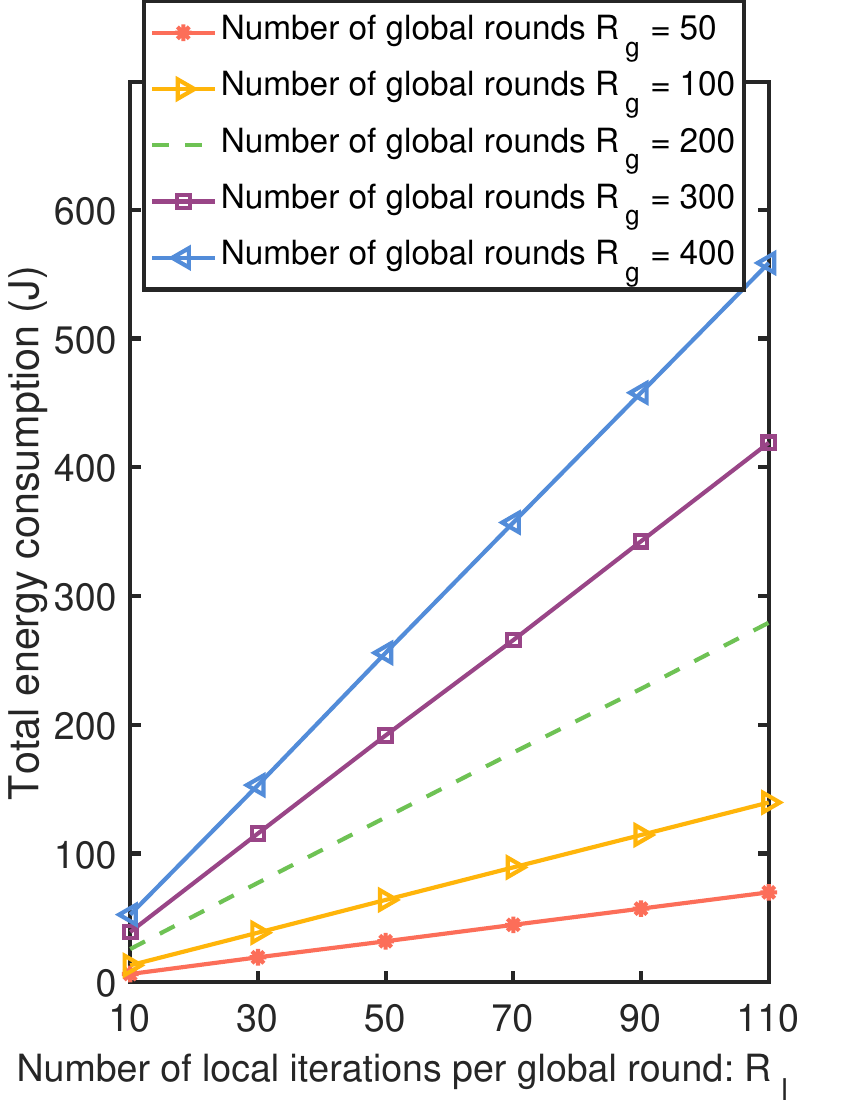}\vspace{-6pt}
    \caption{Total energy consumption}
    \label{fig:total_e_different_iter}
\end{subfigure} \hspace{-15pt}
\begin{subfigure}{.25\textwidth}
  \centering
  \includegraphics[width=1\linewidth]{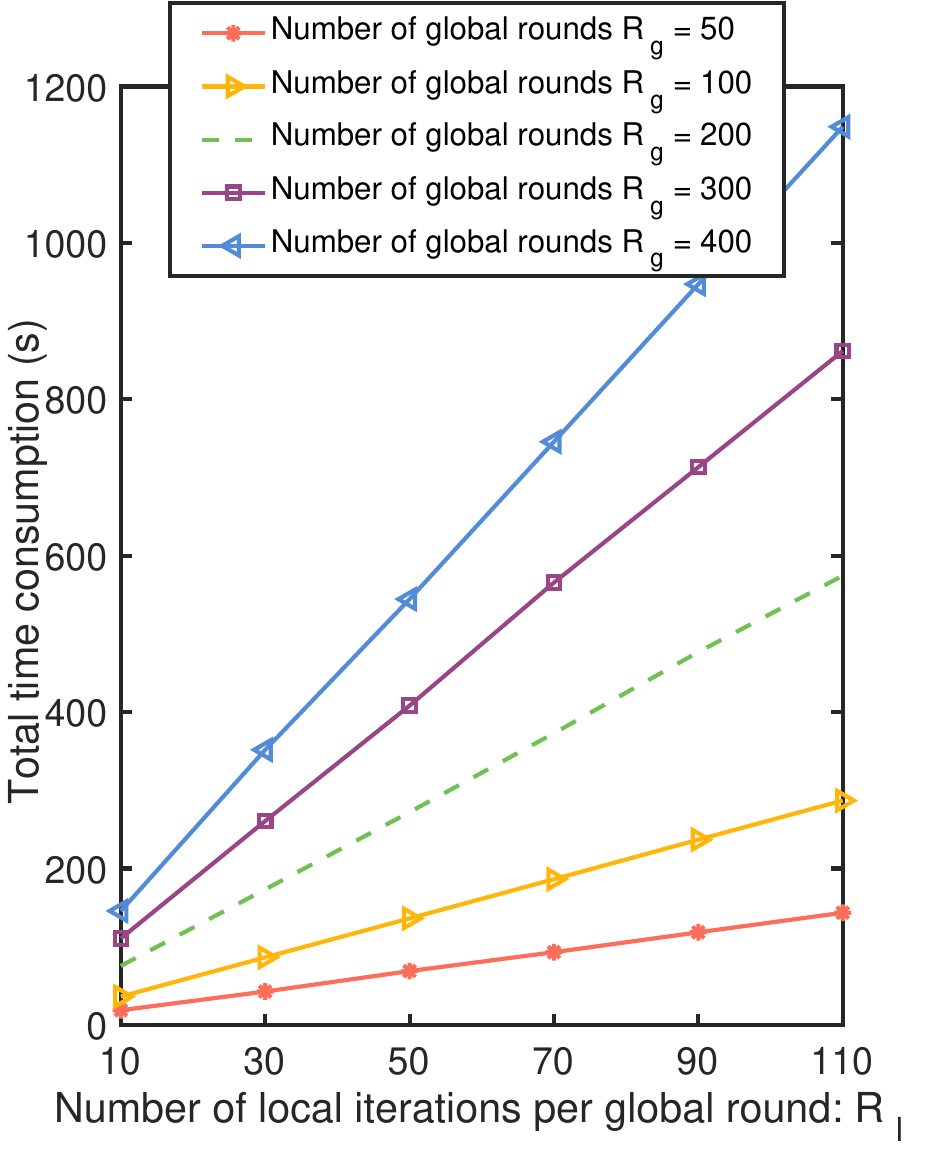}\vspace{-6pt}
    \caption{Total delay\vspace{-6pt}}
    \label{fig:total_time_different_iter}
\end{subfigure}\vspace{-6pt}
\caption{Experiments with different computation rounds. $w_1 = w_2 = 0.5$.\vspace{-15pt}}
\end{figure}


Apparently, it can be observed from Fig. \ref{fig:joint_commu_comp_comparison} that our proposed algorithm performs better than the other two schemes. In addition, transmission energy takes the lead of the total energy consumption. Only optimizing the computation energy cannot exceed the performance of only optimizing the transmission energy. 
Furthermore, we are able to discover the conflicting relationship between maximum completion time and total energy consumption. As maximum completion time increases, there will be better solutions to reduce total energy consumption. However, if we relax the constraint of the maximum completion time $T$ too much, the optimal solution will reach the minimum limit of transmission power and CPU frequency. That is the reason why the gap of these three lines becomes small when $T$ is large.

\subsection{{\color{black}Superiority} of Our  Algorithm over Existing Work}
{\color{black}Prior to our study,  the state of the art for FL with FDMA     is \cite{yang2021energy} by Yang \emph{et~al.} on energy optimization subject to the delay constraint.}
In the comparison below, we refer to their method as Scheme 1. Scheme 1 has two more optimization variables, which are local accuracy and the delay. To be fair, we only use the method (Algorithm 3 in \cite{yang2021energy}) that they use to optimize the transmission power, the bandwidth and the CPU frequency of each user. Since the optimization of the local accuracy and delay is a separate section in Scheme 1, it will not affect the comparison between our method and Scheme 1.

The maximum completion time is not included in the objective function of the optimization problem in Scheme~1. Instead, it appears as a constraint. Therefore, to be fair, we set a fixed maximum completion time $T$ to compare our optimization algorithm with Scheme 1. With a fixed maximum completion time $T$, our minimization problem (\ref{equa:min2}) is under the condition that $w_1 = 1$ and $w_2 = 0$. Besides, the initial transmission power and bandwidth of device $n$ are set as $p_n = p^{max}$, $B_n = \frac{B}{2N}$. The comparison result is shown in Fig. \ref{fig:comparison_e}.
From Fig. \ref{fig:comparison_e}, we can find that each line of our proposed algorithm is below the corresponding line of Scheme 1, which reveals that our algorithm's performance is better. Additionally, as the maximum completion time $T$ decreases, the total energy consumption gap between our algorithm and Scheme 1 gets larger, thereby indicating that our proposed algorithm leads to a better solution when the maximum completion time $T$ is small. Therefore, our proposed algorithm will show more superiority for the training scenario with a restricted delay.

At $T_3 = 150$ s, the line of our algorithm first decreases and then slightly increases. This is because {\color{black}our initial transmission power is initialized as $p_n = p^{max}$  at the start of our optimization algorithm}. After finding the optimal solution at a specific $p^{max}$, the subsequent total energy consumption will slightly increase due to the different starting point from the previous step as $p^{max}$ continues increasing.

\begin{figure}
\centering
\hspace{-5pt}\begin{minipage}{.21\textwidth}
  \centering
  \includegraphics[width=1\linewidth]{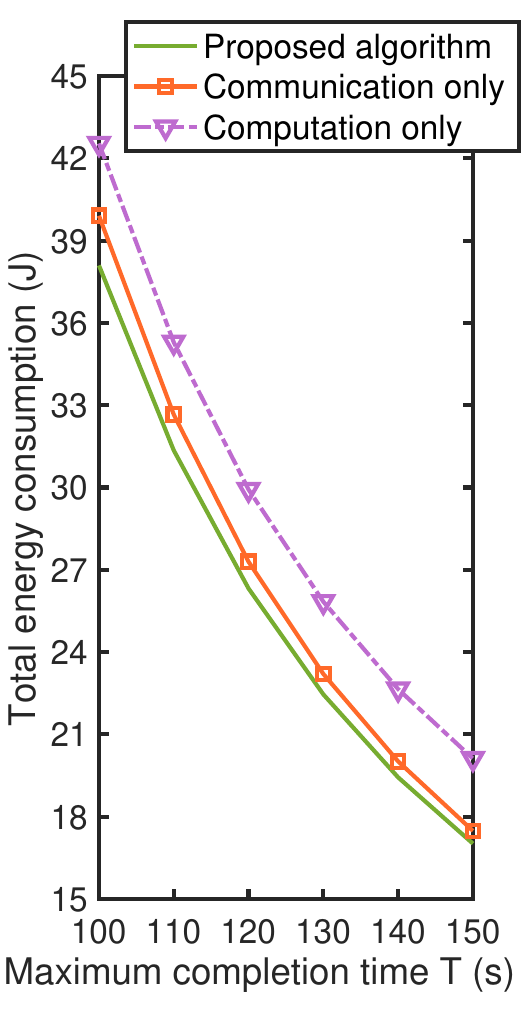}\vspace{-6pt}
    \caption{Total energy consumption at different maximum completion time limits.\vspace{-10pt}}
    \label{fig:joint_commu_comp_comparison}
\end{minipage}\hspace{-2pt}
\begin{minipage}{.285\textwidth}
  \centering
 \includegraphics[width=1\linewidth]{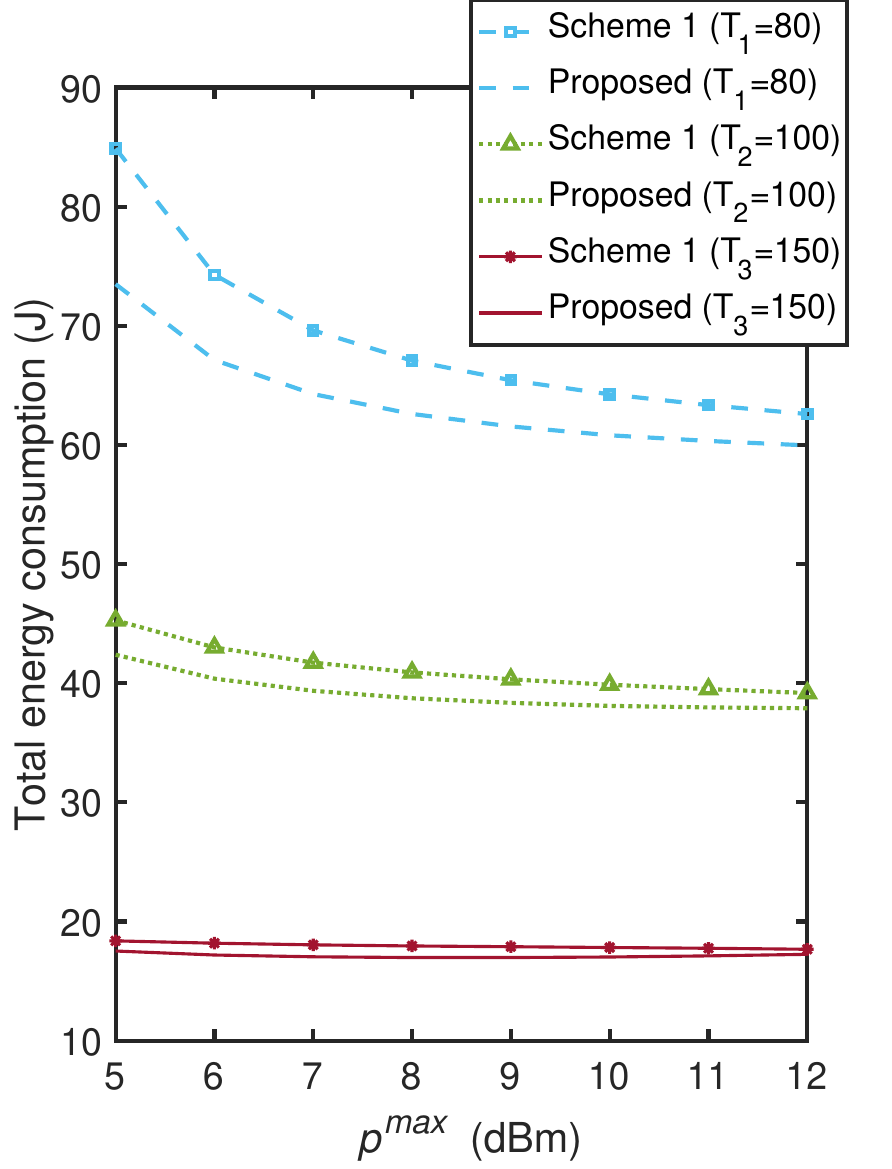}\vspace{-6pt}
    \caption{Total energy consumption at different maximum transmit power limits with fixed maximum completion time $T_1 = 80$ s, $T_2 = 100$ s and $T_3 = 150$ s. The number of users is $N = 50$.\vspace{-10pt}}
    \label{fig:comparison_e}
\end{minipage}
\end{figure}

%
%


{

\renewcommand{\baselinestretch}{.95}


}
\renewcommand{\baselinestretch}{1}

\normalsize

\renewcommand{\theequation}{A.\arabic{equation}}
  \setcounter{equation}{0}  

\section*{Appendix A: Proof of Lemma \ref{lemma:requisite_fra_prog}} \label{appen:proof_lemma_requis}
It is obvious that $K(p_n, B_n) = p_nd_n$ is an affine function with $B_n=0$. Thus, $K(p_n, B_n)$ is convex.


For any vector $\bm{x}=[x_{1}, x_{2}]^T \in \mathbb{R}^{2}$, we have
$
   \bm{x}^THessian(G)\bm{x} = -\frac{(x_{1}g_nB_n-x_{2}g_np_n)^2}{B_n^3N_0^2(\frac{g_np_n}{B_nN_0}+1)^2 \ln{2}} \le 0.$
Because $\bm{x}^T*Hessian(G)*\bm{x} \le 0$ for all $\bm{x}$, $Hessian(G)$ is a negative semidefinite matrix. Therefore, $G(p_n, B_n)$ is a concave function. \textbf{Lemma \ref{lemma:requisite_fra_prog}} is proved.

\section*{Appendix B: Proof of Theorem \ref{theor:express_B_p}} \label{appen:proof_express_bp}
We derive the relationship between $\bm{p}$ and $\bm{B}$ from (\ref{lag:partial_p}), so 
\begin{align} \label{equa:p_B_relation}
  \textstyle{  p_n = (\frac{(\nu_n\beta_n+\tau_n)g_n}{N_0d_n\nu_n\ln2}-1)\frac{N_0B_n}{g_n}}.
\end{align}
From (\ref{lag:data_rate}) and (\ref{equa:p_B_relation}), we have
$B_n = \frac{r_n^{min}}{\log_2(\frac{(\nu_n\beta_n+\tau_n)g_n}{N_0d_n\nu_n\ln2})},~if~ \tau_n \neq 0. $.
Substituting (\ref{equa:p_B_relation}) in (\ref{equal:sp2_v2_lagran}), we replace $p_n$ by $B_n$ in the Lagrangian function:
$
L_3(B_n,  \tau_n, \mu) = B_n(\frac{\nu_n\beta_n+\tau_n}{\ln2}-\frac{\nu_nd_nN_0}{g_n}-(\nu_n\beta_n+\tau_n)\times   \log_2(\frac{(\nu_n\beta_n+\tau_n)g_n}{N_0d_n\nu_n\ln2})+\mu)+\sum_{n=1}^N r_n^{min}* \tau_n-\mu B.$
The corresponding dual problem writes as 
\begin{align} \label{equa:subp2_dual}
\max_{\tau_n, \mu} ~&  \textstyle{g(\tau_n, \mu) = \sum_{n=1}^N r_n^{min}*\tau_n-\mu B} \\
\text{subject to}~& \textstyle{\frac{a_n}{\ln2}\hspace{-1pt}-\hspace{-1pt}j_n\hspace{-1pt}-\hspace{-1pt}a_n\log_2(\frac{a_n}{j_n\ln2})\hspace{-1pt}+\hspace{-1pt}\mu = 0}, \mu \ge 0, ~\tau_n \ge 0, \label{subp2_constra1} 
\end{align} 
where $a_n = \nu_n\beta_n+\tau_n$ and $j_n = \frac{\nu_nd_nN_0}{g_n}$.
Besides, constraint (\ref{subp2_constra1}) derives the relationship between $\tau_n$ and $\mu$: 
\begin{align}\tau_n = \textstyle{\frac{(\mu-j_n)\ln2}{W(\frac{\mu-j_n}{e*j_n})}-\nu_n\beta_n,} \label{tau_mu_relation}
\end{align}
where $W(\cdot)$ is Lambert $W$ function and $e$ is the mathematical constant. 
Given (\ref{tau_mu_relation}), the dual function is simplified as 
\begin{align}
\textstyle{g(\mu) = \sum_{n=1}^N r_n^{min}*(\frac{(\mu-j_n)\ln2}{W(\frac{\mu-j_n}{e*j_n})}-\nu_n\beta_n)-\mu B.}
\end{align}
Take the first derivative and we get $g'(\mu) = \sum_{n=1}^N\frac{r_n^{min}\ln2}{W(\frac{\mu-j_n}{e*j_n})}(1-\frac{1}{W(\frac{\mu-j_n}{e*j_n})+1})-B$. 
Additionally, $g''(\mu)\hspace{-1pt}=\hspace{-1pt}-\sum_{n=1}^N \frac{r_n^{min}\ln2W(\frac{\mu-c_n}{e*c_n})}{(\mu-j_n)(1+W(\frac{\mu-j_n}{e*j_n}))^3} \hspace{-1pt}\le \hspace{-1pt}0$ due to $\frac{\mu-j_n}{e*j_n}\hspace{-1pt}\ge\hspace{-1pt} -\frac{1}{e}$.

Therefore, $g'(\mu)$ is a monotone decreasing function, and $g(\mu)$ is concave, so $g(\mu)$ reaches maximum when $g'(\mu) = 0$. Bisection method can be used to find $\mu$ satisfying $g'(\mu) = 0$. Naturally, $\tau_n = \max( \text{(\ref{tau_mu_relation})}, 0)$.
Note $\tau_n \neq 0$ implies $B_n = \frac{r_n^{min}}{\log_2(\frac{(\nu_n\beta_n+\tau_n)g_n}{N_0d_n\nu_n\ln2})}$. We denote the sum of bandwidth of these devices by $B_{\tau_n\neq0}$ and the set of these devices by $\mathcal{N}_{\tau_n\neq0}$, with $N_{\tau_n\neq0}$ being the cardinality of $\mathcal{N}_{\tau_n\neq0}$.

We use (\ref{equa:p_B_relation}) to replace $p_n$ by $B_n$ in problem \textbf{\textit{SP2\_v2}} of (\ref{equa:subproblem2_lagrangian_min}), and remove those devices whose bandwidth is calculated in the previous step and numbering the left devices. Then the new problem becomes
\begin{align} \label{SP2_v3}
   & {\min_{B_n}\sum_{n=1}^{N-N_{\tau_n \neq 0}}  \hspace{-2pt}(\frac{\nu_n\beta_n}{\ln2}\hspace{-1pt}-\hspace{-1pt}\frac{N_0d_n\nu_n}{g_n}\hspace{-1pt}-\hspace{-1pt}\nu_n\beta_n  \log_2(\frac{\beta_ng_n}{N_0d_n\ln2}))B_n} \\
\hspace{-1pt}&\text{subject to}, \notag \\
 & { p_n^{min} \le (\frac{\beta_ng_n}{N_0d_n\ln2}-1)\frac{N_0B_n}{g_n} \le p_n^{max}, ~\forall n   \in \mathcal{N}\backslash\mathcal{N}_{\tau_n\neq0},} \nonumber \\
& {\sum_{n=1}^{N-N_{\tau_n\neq0}} B_n \le B-B_{\tau_n\neq0},}\nonumber 
\end{align}
where $\mu$ and $\tau_n$ have already been solve in previous steps.

Then the problem \textbf{\textit{SP2\_v2}} becomes a convex optimization problem with just one variable. Thus, we can use the convex problem solver CVX \cite{grant2014cvx} to solve problem (\ref{SP2_v3}).








\end{document}